\documentclass{article} 
\usepackage{iclr2025_conference,times}

\usepackage[T1]{fontenc}      
\usepackage[utf8]{inputenc}   

\usepackage{hyperref}
\usepackage{amsmath}
\usepackage{amsthm}
\usepackage{cleveref}
\usepackage{amssymb,mathrsfs} 

\usepackage{amsfonts}
\usepackage{upgreek}
\usepackage{graphicx}
\usepackage{color}

\usepackage[inline]{enumitem}

\usepackage{url}

\usepackage{tikz}

\usepackage{xargs}
\usepackage{caption}

\usepackage{wrapfig}

\usepackage{booktabs}

\usepackage{todonotes}

\usepackage{autonum}

\usepackage{algorithm}
\usepackage{algpseudocode}
\usepackage{dsfont}

\newtheorem{theorem}{Theorem}

\newtheorem{proposition}{Proposition} 

\newtheorem{lemma}{Lemma}

\crefname{theorem}{theorem}{theorems}
\Crefname{Theorem}{Theorem}{Theorems}

\crefname{remark}{remark}{remarks}
\Crefname{Remark}{Remark}{Remarks}

\crefname{proposition}{proposition}{propositions}
\Crefname{Proposition}{Proposition}{Propositions}

\crefname{corollary}{corollary}{corollaries}
\Crefname{Corollary}{Corollary}{Corollaries}

\crefname{example}{example}{examples}
\Crefname{Example}{Example}{Examples}

\crefname{lemma}{lemma}{lemmas}
\Crefname{Lemma}{Lemma}{Lemmas}

\crefname{proof}{example}{proofs}

\crefname{figure}{figure}{figures}
\Crefname{Figure}{Figure}{Figures}

\Crefname{assumption}{\textbf{H}\hspace{-3pt}}{\textbf{H}\hspace{-3pt}}
\crefname{assumption}{\textbf{H}}{\textbf{H}}

\Crefname{assumptionAO}{\textbf{AO}\hspace{-3pt}}{\textbf{AO}\hspace{-3pt}}
\crefname{assumptionAO}{\textbf{AO}}{\textbf{AO}}

\Crefname{assumptionL}{\textbf{L}\hspace{-3pt}}{\textbf{L}\hspace{-3pt}}
\crefname{assumptionL}{\textbf{L}}{\textbf{L}}

\Crefname{assumptionA}{\textbf{A}\hspace{-3pt}}{\textbf{A}\hspace{-3pt}}
\crefname{assumptionA}{\textbf{A}}{\textbf{A}}

\Crefname{assumptionG}{\textbf{G}\hspace{-3pt}}{\textbf{G}\hspace{-3pt}}
\crefname{assumptionG}{\textbf{G}}{\textbf{G}}

\Crefname{assumptionAp}{\textbf{A'}\hspace{-3pt}}{\textbf{A'}\hspace{-3pt}}
\crefname{assumptionAp}{\textbf{A'}}{\textbf{A'}}

\hypersetup{
  colorlinks = true,
  citecolor = blue,
}

\def\rset{\mathbb{R}}

\def\rmd{\mathrm{d}}

\def\rme{\mathrm{e}}
\def\rmi{\mathrm{i}}

\newcommandx{\functionspace}[2][1=+]{\mathbb{F}_{#1}(#2)}

\newcommand{\LeftEqNo}{\let\veqno\@@leqno}

\newcommand{\N}{\ensuremath{\mathbb{N}}}

\newcommand{\PE}{\mathbb{E}}
\newcommand{\PP}{\mathbb{P}}

\newcommandx{\Vnorm}[2][1=V]{\| #2 \|_{#1}}
\newcommandx{\norm}[2][1=]{\ifthenelse{\equal{#1}{}}{\left\Vert #2 \right\Vert}{\left\Vert #2 \right\Vert^{#1}}}
\newcommandx{\normLigne}[2][1=]{\ifthenelse{\equal{#1}{}}{\Vert #2 \Vert}{\Vert #2\Vert^{#1}}}

\newcommand{\defEns}[1]{\left\lbrace #1 \right\rbrace }
\newcommand{\defEnsLigne}[1]{\lbrace #1 \rbrace }

\newcommandx\probaMarkovTilde[2][2=]
{\ifthenelse{\equal{#2}{}}{{\widetilde{\mathbb{P}}_{#1}}}{\widetilde{\mathbb{P}}_{#1}\left[ #2\right]}}

\newcommand{\expe}[1]{\PE \left[ #1 \right]}

\def\ie{\textit{i.e.}}

\def\eqsp{\;}

\newcommand{\ocint}[1]{\left(#1\right]}

\newcommandx{\weight}[2][2=n]{\omega_{#1,#2}^N}

\newcommandx\sequence[3][2=,3=]
{\ifthenelse{\equal{#3}{}}{\ensuremath{\{ #1_{#2}\}}}{\ensuremath{\{ #1_{#2}, \eqsp #2 \in #3 \}}}}
\newcommandx\sequenceD[3][2=,3=]
{\ifthenelse{\equal{#3}{}}{\ensuremath{\{ #1_{#2}\}}}{\ensuremath{( #1)_{ #2 \in #3} }}}

\newcommandx{\sequencen}[2][2=n\in\N]{\ensuremath{\{ #1_n, \eqsp #2 \}}}
\newcommandx\sequenceDouble[4][3=,4=]
{\ifthenelse{\equal{#3}{}}{\ensuremath{\{ (#1_{#3},#2_{#3}) \}}}{\ensuremath{\{  (#1_{#3},#2_{#3}), \eqsp #3 \in #4 \}}}}
\newcommandx{\sequencenDouble}[3][3=n\in\N]{\ensuremath{\{ (#1_{n},#2_{n}), \eqsp #3 \}}}

\def\iid{\text{i.i.d.}}

\def\eg{e.g.}

\newcommand{\opnorm}[1]{{\left\vert\kern-0.25ex\left\vert\kern-0.25ex\left\vert #1 
    \right\vert\kern-0.25ex\right\vert\kern-0.25ex\right\vert}}

\newcommandx{\CPE}[3][1=]{{\mathbb E}_{#1}\left[\left. #2 \, \middle \vert \, #3 \right. \right]} 

\newcommandx{\CPELigne}[3][1=]{{\mathbb E}_{#1}[\left. #2 \,  \vert \, #3 \right. ]} 

\newcommandx{\CPVar}[3][1=]{\mathrm{Var}^{#3}_{#1}\left\{ #2 \right\}}
\newcommand{\CPP}[3][]
{\ifthenelse{\equal{#1}{}}{{\mathbb P}\left(\left. #2 \, \right| #3 \right)}{{\mathbb P}_{#1}\left(\left. #2 \, \right | #3 \right)}}

\newcommandx{\osc}[2][1=]{\mathrm{osc}_{#1}(#2)}

\newcommand\coupling[2]{\Gamma(\mu,\nu)}

\renewcommand{\geq}{\geqslant}
\renewcommand{\leq}{\leqslant}

\def\mtt{\mathtt{m}}

\newcommand{\txts}{\textstyle}

\def\gauss{\mathcal{N}}
\def\Idd{\mathrm{I}_d}

\def\bfo{\mathbf{o}}

\newcommandx{\Voi}[1][1=i]{\mathfrak{V}_{\bfo,#1}}
\newcommandx{\Vlyapc}[2][1=\bfo,2=i]{\mathfrak{V}_{#1,#2}}

\def\Wlyap{\mathfrak{W}}
\def\bfomega{\boldsymbol{\omega}}

\newcommandx{\Woi}[1][1=i]{\Wlyap_{\bfomega,#1}}
\newcommandx{\Wlyapc}[2][1=\bfomega,2=i]{\Wlyap_{#1,#2}}

\newcommand{\gstep}[1][t]{\gamma_{#1}} 
\newcommand{\gall}[1][t]{\gamma_{1 \to #1}} 
\newcommand{\sstep}[1][t]{\sigma_{#1}} 
\newcommand{\sall}[1][t]{\sigma_{1 \to #1}} 
\newcommand{\Sall}[1][t]{\Sigma_{1 \to #1}} 
\newcommand{\SallA}[1][t]{\Sigma_{1 \to #1}(A_{1:#1})} 
\newcommand{\Salla}[1][t]{\Sigma_{1 \to #1}(a_{1:#1})} 
\newcommand{\GammaA}[1][t]{\Gamma_{#1}(A_{1:#1})} 
\newcommand{\Gammaa}[1][t]{\Gamma_{#1}(a_{1:#1})} 
\newcommand{\barA}[1][0,1]{\bar{A}^t_{#1}}

\def\SallPt{\Sigma'_{t}(\barA)} %
\def\SallPtm{\Sigma'_{t-1}(\barA[0])} %
\def\GammaP{\Gamma'_t(\barA)} %

 \newcommand{\KL}{\mathrm{KL}}
 \newcommand{\pstar}{p_{\star}}

\newcommand{\mR}{\mathbb{R}}

\newcommand{\mE}{\mathbb{E}}

\newcommand{\eqd}{\,{\buildrel d \over =}\,}

\newcommand{\normal}{\gauss}
\newcommand{\stable}{\mathcal{S}}

\usepackage{xparse} 

\newcommand{\kArg}[1]{k_{#1}^{(\alpha)}}
\newcommand{\kTwoArg}[1]{k_{#1}^{(2)}}

\def\ik{\overleftarrow{k}}
\newcommand{\ikArg}[1]{\ik_{#1}^{(\alpha)}}

\newcommand{\ikTwoArg}[1]{\ik_{#1}^{(2)}}

\def\iq{\overleftarrow{q}}
\newcommand{\qArg}[1]{\iq_{#1}^{\theta}}

\newcommand{\pArg}[1]{p_{#1}^{(\alpha)}}

\def\pt{\pArg{t}}
\def\iX{\overset{\leftarrow}{X}}

\def\lossDDPM{\mathscr{L}^{\mathrm{D}}}
\def\lossDLPM{\mathscr{L}^{\mathrm{L}}}
\def\lossDLPMD{\mathscr{L}^{\mathrm{L}}}
\def\lossDLPMLess{\mathscr{L}^{\mathrm{Less}}}
\def\lossDLPMSimple{\mathscr{L}^{\mathrm{Simple}}}
\def\lossDLPMSimpleLess{\mathscr{L}^{\mathrm{SimpleLess}}}
\def\meanq{\hat{\mtt}^{\theta}}
\def\Sigmaq{\hat{\Sigma}^{\theta}}
\def\epsilonq{\hat{\epsilon}^{\theta}}

\def\meank{\Tilde{\mtt}}
\def\Sigmak{\Tilde{\Sigma}}

\def\epsa{\epsilon^{(\alpha)}}

\def\uppsia{\uppsi^{(\alpha)}}

\def\Cst{\mathrm{Cst}}

\def\stableA{\stable_{\alpha/2,1}(0, c_A )}

\def\stablesym{\stable^{\text{i}}_{\alpha}(0, \Idd)}

\def\levy{L^{\alpha}}
\def\score{S^{(\alpha)}}

\def\dlpmm{DLPM$_{5}$}
\def\dlimm{DLIM$_{5}$}
\def\fpr{F_1^{\text{pr}}}

\newcommand{\hArg}[1]{h_{#1}^{(\alpha)}}
\newcommand{\listT}[1]{\{#1\}_{t=0}^T}

\def\iY{\overleftarrow{Y}}
\def\iYArg{\overleftarrow{Y}^{\theta}}
\def\iZArg{\overleftarrow{Z}^{\theta}}
\def\iXArg{\overleftarrow{X}^{\theta}}
\def\listTA{\{A_t\}_{t=1}^T}
\def\listTG{\{G_t\}_{t=1}^T}

\def\Cauchy{\text{Cauchy}}
\def\LossCauchy{\mathscr{L}^{\text{Cauchy}}}

\title{Heavy-Tailed Diffusion with Denoising L\'evy Probabilistic Models}

\author{
    Dario Shariatian\textsuperscript{1}, Umut Simsekli\textsuperscript{1}, Alain Durmus\textsuperscript{2} \\
    \textsuperscript{1}INRIA - Department of Computer Science, PSL Research University, Paris, France \\
    \textsuperscript{2}\'Ecole Polytechnique - CMAP, IP Paris, Palaiseau, France \\
    \texttt{\{dario.shariatian, umut.simsekli\}@inria.fr}\\ 
    \texttt{alain.durmus@polytechnique.edu}
}

\iclrfinalcopy 

\begin{document}

\maketitle
\begin{abstract}

Exploring noise distributions beyond Gaussian in diffusion models remains an open challenge. While Gaussian-based models succeed within a unified SDE framework, recent studies suggest that heavy-tailed noise distributions, like $\alpha$-stable distributions, may better handle mode collapse and effectively manage datasets exhibiting class imbalance, heavy tails, or prominent outliers. 
Recently, Yoon et al.\ (NeurIPS 2023), presented the L\'evy-It\^o model (LIM), directly extending the SDE-based framework to a class of heavy-tailed SDEs, where the injected noise followed an $\alpha$-stable distribution, a rich class of heavy-tailed distributions. However, the LIM framework relies on highly involved mathematical techniques with limited flexibility, potentially hindering broader adoption and further development. 
In this study, instead of starting from the SDE formulation, we extend the denoising diffusion probabilistic model (DDPM) by replacing the Gaussian noise with $\alpha$-stable noise. By using only elementary proof techniques, the proposed approach, \emph{Denoising Lévy Probabilistic Model} (DLPM), boils down to vanilla DDPM with minor modifications. As opposed to the Gaussian case, DLPM and LIM yield different training algorithms and different backward processes, leading to distinct sampling algorithms. 
These fundamental differences translate favorably for DLPM as compared to LIM: our experiments show improvements in coverage of data distribution tails, better robustness to unbalanced datasets, and improved computation times requiring smaller number of backward steps. 

\end{abstract}


\section{Introduction}

The evolution of generative models has introduced several approaches, with diffusion models emerging as one of the most prominent. These models transform a data distribution into a Gaussian distribution via a forward noising process and then learn to reverse it. The foundational work on denoising in this context was presented by \cite{sohldickstein2015deep}, where the goal is to reverse a Markov chain that progressively adds Gaussian noise to the data. This framework culminated in denoising diffusion probabilistic models (DDPM) by \cite{ho2020denoising}, which demonstrated state-of-the-art performance in image generation, while drawing connections to score matching techniques (\cite{song2020generative}). A unified theoretical framework, based on stochastic differential equations (SDEs), further integrated score matching and the denoising framework (\cite{song2021scorebased}). Various generative models build up on this framework, improving its performance (\cite{dhariwal2021diffusion,karras2022elucidating}).

Despite their success, diffusion models exhibit limitations, such as requiring a large number of steps (\cite{ho2020denoising}) and empirically struggling
with imbalanced datasets (\cite{zhang2024longtailed, NEURIPS2023_score_based_levy_lim}), often leading to mode collapse (\cite{dhariwal2021diffusion, deasy2022heavytailed}). 
Tackling heavy-tailed datasets also presents a challenge for diffusion models, as the finite-time diffusion processes generate finite variance data distributions, which are ill-suited for modeling heavy-tailed data. Such datasets, like financial datasets, could benefit from extensions beyond the Gaussian setting, as demonstrated in \cite{stable_borak}.
Moreover, although many techniques exist to improve quality at the expense of diversity (\cite{dhariwal2021diffusionmodelsbeatgans, song2023consistencymodels}), by only using Gaussian noise, the methods typically require highly nontrivial setups in order to obtain improved diversity while maintaining a reasonable quality (\eg, FID score) (\cite{sadat2024cadsunleashingdiversitydiffusion, nobis2024generativefractionaldiffusionmodels}).

Several approaches have been explored to address the limitations of diffusion models, particularly through the use of non-Gaussian, heavy-tailed noise distributions. The motivation behind this is that heavy-tailed distributions, which can take on larger values, may reduce the number of sampling steps and better capture multimodal data distributions by identifying isolated modes through large noise injections (\cite{NEURIPS2023_score_based_levy_lim}). These distributions are also capable of modeling extreme events or rare occurrences in the tails, making them suitable for tasks such as audio generation (\cite{chen2020wavegradestimatinggradientswaveform, kong2021diffwaveversatilediffusionmodel}), where rare but important variations in amplitude or pitch (e.g., prosodies) can enhance sample quality and diversity. An early attempt at using heavy-tailed noise distributions was made by \cite{nachmani2021denoising}, who replaced Gaussian noise with Gamma-distributed noise, reducing the number of diffusion steps and improving data diversity. Similarly, \cite{deasy2022heavytailed} employed generalized Gaussian distributions in a score-matching formulation to improve robustness on unbalanced datasets. However, despite these promising directions, the performance boosts are limited, since both approaches fail to provide a time-reversal formula, either using annealed Langevin dynamics when the score is available, or using heuristics based on the Gaussian approach.

Very recently, partially inspired by \cite{pmlr-simsekli17a,huang2021approximation}, who employ heavy-tailed SDEs in Monte Carlo sampling for challenging distributions, \cite{NEURIPS2023_score_based_levy_lim} extended the SDE framework by replacing the light-tailed Brownian motion with a heavy-tailed driving process, introducing the L\'evy-It\^o model (LIM), improving performance on image data, particularly for unbalanced datasets, offering gains in metrics like FID and diversity. The tail index $\alpha$ controls the degree of tail heaviness, enabling tunable performance based on the characteristics of the data. When $\alpha=2$ the noising process reduces to the Brownian motion (light-tailed), however, whenever $\alpha<2$, the process becomes heavy-tailed with infinite variance (notably, heavier-tailed than the distributions explored by \cite{nachmani2021denoising}, \cite{deasy2022heavytailed}). 

While injecting heavy-tailed, infinite-variance noise might seem natural when the data distribution itself is also heavy-tailed, \cite{NEURIPS2023_score_based_levy_lim} further illustrated that the heavy-tailed noise can also be beneficial when sampling from compactly supported data distributions (\eg, generating images), especially in the presence of class imbalances (i.e., the large `jumps' introduced by the heavy tails can help finding weakly represented modes). Indeed, perhaps being counter-intuitive, it has been shown that a heavy-tailed process can indeed converge to a light-tailed distribution with appropriate care \cite{pmlr-simsekli17a,simsekli2020fractional,huang2020approximationheavytaileddistributionsstabledriven}, making them suitable for sampling from a broad range of data distributions.



\vspace{-8pt}

\paragraph{Motivation.} 
While LIM demonstrates promising results, the technical complexity of the time-reversed SDE presents significant challenges. The Lévy process with $\alpha < 2$ has discontinuous paths and no variance, preventing the use of standard analysis tools. The proof techniques rely on fractional calculus and estimations for pseudo-differential operators, which might not be accessible for the broader community of diffusion-based generative models. While the theory is elegant, we argue that the highly technical nature of LIM, originated due to the use of continuous-time processes, might hinder its development. For instance, it is highly non-trivial to use arbitrary noise schedules. 

Moreover, the loss function used in LIM presents some shortcomings: the theory requires a squared $\ell_2$ loss, assuming the loss remains finite for the considered neural network. However, this may not always hold, as the noise term is heavy-tailed and admits no variance, potentially leading to infinite loss values. As a result, LIM experiments must revert to an $\ell_1$ loss for stable training, suggesting that the original loss function may indeed be unworkable. Additionally, LIM is constrained to isotropic noise, further limiting its flexibility.

To overcome these issues, we propose a simpler yet effective alternative for incorporating heavy-tailed distributions, focusing on a discrete-time framework that leverages more elementary mathematical tools while maintaining performance improvements over Gaussian-based models.

\newpage 
\begin{wrapfigure}{r}{0.55\textwidth}
\centering
\begin{tikzpicture}[thick,scale=0.8, every node/.style={scale=0.8}]
    \node (A) at (0,2) {DDPM };
    \node (B) at (5,2) {Score-based SDE};
    \node (C) at (0,0) {\textbf{DLPM} (This study)};
    \node (D) at (5,0) {LIM (\cite{NEURIPS2023_score_based_levy_lim})};

    \draw[<->] (A) -- (B) node[midway, above] {unified};
    \draw[->] (A) -- (C) node[midway, left] {$\alpha$-stable noise};
    \draw[->] (B) -- (D) node[midway, right] {$\alpha$-stable noise};
    \draw[<->] (C) -- (D) node[midway, above] {\textcolor{red}{\texttimes}};
    \draw[<->, red] (C) -- (D) node[midway, below] {\textcolor{red}{not unified}};
\end{tikzpicture}
\caption{Illustration of available methods.}
\label{fig:chart}
\end{wrapfigure}
\paragraph{Contributions.}
As opposed to LIM which extended the SDE-based framework, here, we take a step back and directly work on the discrete-time DDPM process and replace the Gaussian noise with $\alpha$-stable noise. More precisely, we propose the following Markov process as noising process:
\begin{equation}\label{eq:dlpm_forward}
   X_{t} = \gamma_t X_{t-1} + \sigma_t \epsilon_t^{(\alpha)},
\end{equation}
where $\epsilon_t^{(\alpha)}$ follows a multivariate $\alpha$-stable distribution. When $\alpha=2$, the process recovers the standard Gaussian DDPM, but for $\alpha<2$, it introduces heavy-tailed noise with infinite variance.
A comparison between DLPM and LIM is provided in Figure~\ref{fig:chart}. 
Our contributions are as follows. 
\begin{enumerate}[label=$\bullet$,wide]
    \item \textbf{Simplified mathematical framework.}\ Leveraging a property of stable distributions (cf. Theorem~\ref{thm:main_stable}), we decompose $\epsilon_t^{(\alpha)}$ into a product of a one-dimensional random variable and a Gaussian vector. This transformation reduces the forward process to a Gaussian one with modified scaling, allowing us to approximate the reverse process using only elementary tools.
    We call the resulting generative model Denoising L\'evy Probabilistic Model (DLPM). 

    \item \textbf{Extension to deterministic sampling.}\ Building on the DDIM framework (\cite{DDIM}), we introduce a deterministic sampler for DLPM, termed DLIM, which further reduces the number of sampling iterations, boosting efficiency.

    \item \textbf{Compatibility with existing methods.}\ DLPM maintains compatibility with existing DDPM implementations, requiring only minor modifications, making it a practical and flexible alternative to LIM. This simplicity extends for instance to noise schedules, which are difficult to manipulate in the continuous-time LIM framework.
    
    \item \textbf{Distinct Algorithms from LIM.}\ Unlike the Gaussian case, where DDPM and the score-based SDE formulation are two sides of the same coin (\cite{song2021scorebased}), we show that DLPM and LIM result in distinct training algorithms and backward processes. Importantly, as these heavy-tailed distributions admit no variance, we cannot simply use the square $\ell_2$ loss, problem which DLPM carefully addresses. 
    
    \item \textbf{Improved performances.}\ Thanks to our conditionally Gaussian representation strategy, our networks are only modeling conditional densities given the heavy-tailed variables. Hence the network is not learning a heavy-tailed distribution directly, but a light-tailed conditional distribution, intuitively an easier task. These differences work in favor of DLPM across several performance aspects, particularly where heavy-tailed noise injections are already known to offer advantages. Our experiments show that DLPM provides (i) better coverage of the tails of the data distribution, (ii) improved generation of unbalanced datasets, and (iii) faster computation times, requiring fewer backward steps.

\end{enumerate}

\section{Background on $\alpha$-stable distributions}
\label{subsec:heavy-tailed}







\label{subsec:alpha_stable}
The family of $\alpha$-stable distributions appears as the limiting distribution in the generalized central limit theorem (\cite{gnedenko1968limit}). In the one dimensional case, an $\alpha$-stable distributed random variable $X$ is defined through its characteristic function (\cite{samorodnitsky1996stable}): for $u \in \rset$ 
\begin{equation}
  \label{eq:def_alpha_stable_chara}
  \expe{ \rme^{\rmi u X}}  = 
\exp\defEnsLigne{ \rmi u \mu - |\sigma u|^{\alpha}(1 - \rmi \varphi \beta\, \textrm{sgn}(u))} \eqsp, \quad 
\textrm{where} \ \varphi = \begin{cases}
    \tan (\uppi \alpha/2) &\textrm{if } \alpha \neq 1 \\ 
    -(2/\uppi)\log |\sigma u |&\textrm{otherwise \eqsp.}
\end{cases}
\end{equation}
Here, (i) $\mu \in \rset$ is the location parameter (ii) $\alpha \in \ocint{0,2}$ is the tail index (iii) $\sigma >0$ the scale parameter (iv) $\beta \in [-1, 1]$ determines the right- or left-skewness, and $\textrm{sgn}$ is the sign
function. We denote the $\alpha$-stable distribution by $\stable_{\alpha, \beta}(\mu, \sigma)$.




In the case where $\alpha <1$ and $\beta =1$, the support of the distribution becomes the positive real line (i.e., the random variable is positive), hence we call this distribution `positive stable'. On the other hand, in the case where $ \beta =0$, the distribution
$\stable_{\alpha, 0}(\mu, \sigma)$ is symmetric around $\mu$, and denoted by
$\stable_{\alpha}(\mu, \sigma)$. Furthermore, in the case $\alpha = 2$, the distribution reduces to a Gaussian
$\stable_{\alpha}(\mu, \sigma) = \normal(\mu, 2\sigma^2)$, hence it is light-tailed. 
However, whenever $\alpha < 2$,
$\stable_{\alpha}(\mu, \sigma)$ has heavy tails, \ie, the decay rate of its tail distribution satisfies
$\PP(|X| > r) \sim r^{-\alpha}$ as $r \to  \infty$ (see \cite[Theorem 1.2]{nolan2020univariate}). This implies that $\mathbb{E}[|X|^p ] <\infty$ if and only if $p<\alpha<2$. 

As opposed to Gaussians, there are multiple ways of extending the $\alpha$-stable distributions to the multivariate setting. In this paper, we will be interested in two major cases: (i) the isotropic (also called rotationally invariant)\footnote{The noise distribution used in LIM is the isotropic $\alpha$-stable distribution. Note that our framework allows for different types of $\alpha$-stable distributions by following a single mathematical recipe. } and (ii) the non-isotropic with independent components. 
These distributions are also defined through their respective characteristic functions. The random variable $X \in \mathbb{R}^d$ is isotropic $\alpha$-stable if its characteristic function is given by: for all $u \in \mathbb{R}^d$, $\mE[\exp(iu^\top X)] = \exp(i\mu^\top u - \sigma^\alpha \| u\|^\alpha)$, where $\mu \in \rset^d$ is the location parameter and $\sigma \Idd$ plays the role of a covariance matrix\footnote{in this isotropic case the components are not independent even though the covariance matrix is diagonal.}.  We denote it by $X \sim \stable^{\text{i}}_{\alpha}(\mu, \sigma\Idd)$.
Similarly, $X$ follows the non-isotropic $\alpha$-stable distribution $\stable^{\text{n}}_{\alpha}(\mu, \sigma\Idd)$, if for any $u \in\rset^d$, $\mE[\exp(iu^\top X)] = \exp(i\mu^\top u - \sigma^\alpha \sum_{i=1}^d |u_i|^\alpha)$. While both of these distributions share similar characteristics, such as having power-law tails with the same exponent, the components of the isotropic case are dependent, which results in a significant difference compared to the non-isotropic case, which has independent coordinates. When $\alpha=2$ both options coincide with a multivariate Gaussian.


The following property of stable distributions 
will form the backbone of our algorithm.
\begin{theorem}[See {\cite[Equation 2.5.3]{samorodnitsky1996stable}}]
    \label{thm:main_stable}

Let $\alpha <2$, and let $X \sim \stable^{\text{i}}_{\alpha}(\mu, \sigma\Idd)$. Then, $ X \eqd\mu +  \sigma A^{1/2} {G}$,
    where $\eqd$ denotes equality in distribution, $A \sim \stable_{\alpha/2, 1}(0, c_A)$ is a one-dimensional positive stable random variable with $c_A := \cos^{2/\alpha}(\uppi\alpha/4)$, and ${G} \sim \normal(0,\Idd)$.
\end{theorem}

This theorem shows that a zero-mean, unit-scale isotropic stable random-vector can be equivalently written as the product of a \emph{one dimensional} positive stable random variable and a standard Gaussian random vector. This fundamental property will have a significant impact in terms of incorporating $\alpha$-stable noise in DDPMs in a simple way as, \emph{conditioned on} $A$, the distribution of $X$ is just a Gaussian. 
We conclude this section by noting that a similar decomposition for the non-isotropic case: if $X \sim \stable^{\text{n}}_{\alpha}(\mu, \sigma\Idd) $, then $X \eqd \mu+\sigma \, \mathbf{A}^{1/2} \odot {G}$, where $\odot$ is the component-wise multiplication and  $\mathbf{A}^{1/2} = \{A_i^{1/2}\}_{i=1}^d \in \rset^d$ is a vector with \iid~components with  $A_i \sim \stable_{\alpha/2,1}(0, c_A )$.

\section{Denoising L\'{e}vy Probabilistic Models}
\label{sec:dlpm}

In this section, we develop our algorithm by following a similar route to the development on DDPM: we identify the forward and backward processes associated with \eqref{eq:dlpm_forward} and construct a variational approximation for the backward chain for sampling. 
Here, we focus on isotropic $\alpha$-stable noise; however,  adaptation to non-isotropic $\alpha$-stable noise is straightforward, as all our derivations rely on Theorem~\ref{thm:main_stable}, hence it is omitted.


\subsection{Markovian forward process}
Recall that DLPM is based on the following recursion (restatement of \eqref{eq:dlpm_forward}):
\begin{equation}
\label{eqn:dlpm_recall}
   X_0 \sim \pstar \eqsp, \quad \text{ and } \quad X_{t} = \gstep[t] X_{t-1} + \sstep[t] \epsilon_{t}^{(\alpha)},
\end{equation}
where $\pstar$ is the data distribution and $\{\epsilon_t^{(\alpha)}\}_{t=1}^T$ are independent and distributed as $\stable^{\text{i}}_{\alpha}(0, \Idd)$. Thanks to the `stability' property of $\alpha$-stable distributions, \ie, the sum of two $\alpha$-stable random variables is still $\alpha$-stable (see Appendix~\ref{app:alpha_stable_noise} for details), we can explicitly characterize the conditional distribution of $X_t$ given $X_0$. Setting for any $t \in \{1,\ldots,T\}$,  $\gall := \prod_{i = 1}^t \gstep$, and $\sall := (\sum_{i = 1}^t (\gall[t]\sstep[i]/\gall[i])^{\alpha})^{1/\alpha}$, we show in Proposition~\ref{app:gall_sall} that:
\begin{equation}
    X_{t} \eqd \gall X_0 + \sall \varepsilon_t \eqsp,
\end{equation}
where $\varepsilon_t \sim \stable^{\text{i}}_{\alpha}(0, \Idd)$. 
Similarly to DDPM, the noising schedule parameters $\{(\gamma_t, \sigma_t)\}_{t=1}^{T}$ and the horizon $T$ are set so that the final distribution of $X_T$ is approximately equal to $\stable^{\text{i}}_{\alpha}(0, \sall\Idd)$, choosing either (i) $\gall \to 0$ as $t$ increases, or (ii) $\gall = 1$ and $\sall$ increasing with $t$.
Following the terminology used in DDPM\footnote{In the DDPM literature, these noising schedules are referred to as variance preserving, or variance exploding. We use the term `scale' instead of the variance here, since 
the variance does not exist when $\alpha <2$.}, we refer to schedule (i) as scale preserving and schedule (ii) as scale exploding.

\subsection{Generative process}
Once the forward process is run for large enough time-steps $T$, it is clear that $X_T$ will be approximately stable-distributed. Hence, to go back to the data distribution $p_\star$, we now need to time-revert the forward process so that the reversed process can take some $\alpha$-stable noise and generate a sample from $p_\star$, which is our ultimate goal. More precisely, we aim to find a backward process associated to the Markov chain $\{X_t\}_{t=0}^T$, \ie, a Markov chain $\{\iX_t\}_{t=0}^T$ such that the two processes $\{\iX_{T-t}\}_{t=0}^T$ and $\{X_t\}_{t=0}^T$ have the same distributions. 
Since, by \cite[Theorem 1.9]{Nolan2010StableDistributions}, any non-degenerate $\alpha$-stable distribution has a smooth density with respect to the Lebesgue measure, the Markov chain \eqref{eqn:dlpm_recall} admits a transition density denoted by $\kArg{t|t-1}$.
In addition, $X_t$ admits a density as well, denoted by $\pt$.  Then, it can be easily verified that a Markov process starting from $p_{T}^{(\alpha)}$ and with transition densities, for $t \in \{0,\ldots,T-1\}$, $\ikArg{t-1|t}(x_{t-1}|x_{t}) \propto p_{t-1}^{(\alpha)}(x_t) \kArg{t | t-1}(x_{t},x_{t-1})$ for any $x_{t-1},x_{t}$, is a backward process associated with $\{X_t\}_{t=0}^T$, where $\propto$ denotes equality up to a normalization constant. As in the case of DDPM, this backward transition densities are unfortunately intractable, hence, we will develop a variational scheme for their approximation.

\subsubsection{Approximation of the backward transition densities in DDPM}

To ease the introduction of our approach, let us first recall the strategy 
taken in DDPM, which approximates the backward kernels by relying on a variational approximation for $\ikArg{0:T}(x_{0:T}) := \pArg{T}(x_T)\prod_{t=T}^{1} \ikArg{t-1|t}(x_{t-1}|x_{t})$, where $\alpha =2$ and  $x_{0:T} := (x_0,\ldots,x_T) \in \rset^{d\times (T+1)}$. 
More precisely, the goal is to find the `closest' distribution to $\ikArg{0:T}$ in a family of distributions $\{\qArg{0:T} \,: \, \theta \in\Theta\}$, indexed by a parameter $\theta$ taking values 
in some parameter space $\Theta$ (typically taken as a neural network). 

The variational family is assumed to have the same decomposition as $\ikArg{0:T}(x_{0:T})$, thus such that $\qArg{0:T}(x_{0:T}) := \qArg{T}(x_T) \prod_{t=T}^{1} \qArg{t-1|t}(x_{t-1}|x_{t})$, where $\qArg{T}$ is chosen to be the density of $\normal(0,\sall^2 \Idd)$ as an approximation of $\pArg{T}$.
Then, $\theta$ is obtained by minimizing the following objective function \cite[Equation 5]{ho2020denoising}:
\begin{align}
\label{eqn:loss_ddpm}
\lossDDPM(\theta) := \sum_{t=2}^{T} \lossDDPM_{t-1}(\theta) \quad \text{with} \quad  \lossDDPM_{t-1}(\theta) = \PE[\KL(\kArg{t-1|0,t}(\cdot|X_0,X_{t}) \| \qArg{t-1|t}(\cdot|X_{t}))] \eqsp,  
\end{align}
where $\KL$ denotes the Kullback-Leibler divergence and $\kArg{t-1|0,t}$ denotes the conditional density of $X_{t-1}$ given $X_0$ and $X_{t}$. As $\alpha =2$ in this case, $\kArg{t-1|0,t}$ is Gaussian \cite[Equation 6,7]{ho2020denoising}, motivating the choice of Gaussian densities $\qArg{t-1|t}(x_{t-1} | x_{t})$ as elements of the variational family at hand, since one obtains a closed-form formula for the $\KL$ terms, \ie,
\begin{equation}
\label{eq:gaussian_variational}
    \qArg{t-1|t}(x_{t-1} | x_{t}) = \upphi_d \left(x_{t-1} | \meanq_{t-1}(x_{t}), \Sigmaq_{t-1}(x_{t}) \right) \eqsp,
  \end{equation}
  where $(x,\mtt,\Sigma) \mapsto  \upphi_d(x|\mtt,\Sigma)$ is the density of the $d$-dimensional Gaussian distribution with mean $\mtt$ and covariance matrix $\Sigma$, and $\meanq_{t-1},\Sigmaq_{t-1}$ are functions of $x_{t}$ parameterized by $\theta$.
This approach relies on the fact that $\kArg{t-1| 0, t}$ is analytically tractable. Unfortunately, when $\alpha < 2$, it is not the case anymore. We now expose our methodology to address this limitation. 

  %

  

\subsubsection{A data augmentation approach} \label{data_augmentation_approach}
To obtain a tractable objective function for learning a variational approximation of the backward transition densities, we rely on a data augmentation approach, which is a classical MCMC technique (see \cite{brooks2011handbook}, Chapter 10). Consider the Markov chain $Y_0 \sim \pstar \eqsp$, and for $t\in\{1,\ldots,T\} $, 
\begin{equation}
  Y_{t}  = \gamma_t Y_{t-1} + \sigma_t A_{t}^{1/2} G_{t}  \eqsp,
    \label{eq:forward_t_with_a}
\end{equation}
where $\listTG$ and $\listTA$ are independent random variables, distributed according to $G_t \sim \gauss(0,\Idd)$, and $A_t \sim \stable_{\alpha/2,1}(0, c_A )$ with $c_A =  \cos^{2/\alpha}(\uppi\alpha/4)$.
From Theorem~\ref{thm:main_stable}, $\{Y_t\}_{t=0}^T$ is a Markov chain that admits the same distribution as $\{X_t\}_{t=0}^T$. 
As a result, \emph{conditioned on} $\{A_t\}_{t=1}^T$ and $Y_0$, $\{Y_t\}_{t=1}^T$ is a Markov chain with Gaussian transition densities:
\begin{equation}
    \kArg{1:T|0, a}(y_{1:T} |y_0,a_{1:T}) = \prod\nolimits_{t=1}^T  \upphi_d(y_t | \gamma_t y_{t-1}, \sigma^2_ta_t)\eqsp.
\end{equation}
Again, 
we can explicitly characterize the conditional distribution of $Y_t$ given $Y_0, \listTA$. Setting for any $t \in \{1,\ldots,T\}$,  $\Salla := \sum_{k = 1}^t (\gall[t]a_k^{1/2}\sstep[k] / \gall[k])^{2}$, we show in \Cref{app:gall_Sall} that:
\begin{equation}
\label{eq:y_t_y_0_a}
    Y_{t} \eqd \gall Y_0 + \SallA \varepsilon_t \eqsp,
\end{equation}
where $\varepsilon_t \sim \stable^{\text{i}}_{\alpha}(0, \Idd)$. 
We propose approximating the backward process associated to $\{Y_t\}_{t=0}^T$, \emph{given} $\{A_t\}_{t=1}^T$, adapting the DDPM
approach.
This time, for the backward process, we use the conditional density of $Y_{t-1}$ given $Y_0, Y_{t}$ and $A_{1:T}$: 
\begin{equation}
    \ikArg{1:T | 0, a}(y_{1:T} | y_0, a_{1:T}) := \pArg{T}(y_T)\prod\nolimits_{t=T}^{2} \kArg{t-1|0,t, a}(y_{t-1}|y_{t}, y_0, a_{1:t}),
\end{equation}
where $\kArg{t-1|0,t, a}(\cdot|y_{t}, y_0, a_{1:t})$ is now the tractable density of a Gaussian distribution:

\begin{proposition}
\label{prop_backward}
The density of the backward process associated to $\listT{Y_t}$ given $Y_0, \listTA$, denoted by
$\kArg{t-1|t,0,a}(\cdot|y_{t}, y_0, a_{1:t})$, is the density of a Gaussian distribution $\normal(\meank_{t-1}, \Sigmak_{t-1})$, with mean $\meank_{t-1}$ and variance $\Sigmak_{t-1}$ equal to:
\begin{equation}
    \meank_{t-1}(y_t, y_0, a_{1:t}) = \frac{1}{\gstep}\left( y_{t} - \Gammaa \sall \epsilon_{t}(y_t, y_0) \right)\eqsp, \ 
    \Sigmak_{t-1}(a_{1:t}) = \Gammaa \Salla[t-1]\eqsp,
\end{equation}
where $\Sall$ is as in \eqref{eq:y_t_y_0_a}, 
$\Gamma_t = 1 - \gstep \Sall[t-1] / \Sall[t]$, and $\epsilon_{t}(y_t, y_0) = (y_{t} - \gall y_0)  / \sall$.
\end{proposition}
See \Cref{app:backward_process} for the derivations required for \Cref{prop_backward}. 
For our generative model, we reconsider the family of Gaussian variational approximation introduced in \eqref{eq:gaussian_variational}, modified to account for an iid.\ sequence $\{A_t\}_{t=1}^T$: $\qArg{0:T}(y_{0:T}) := \int \qArg{T}(y_T) \prod_{t=T}^{1} \qArg{t-1|t, a}(x_{t-1}|x_{t}, a_{1:t}) \uppsia_{1:T}(a_{1:T}) \rmd a_{1:T}$, where $\uppsia_{1:T}(a_{1:T}) = \prod_{t=1}^T \uppsia(a_t)$ and $\uppsia$ denotes the density of $\stable_{\alpha/2,1}(0, c_A )$, and
\begin{equation}
\qArg{t-1|t,a}(y_{t-1} | y_{t}, a_{1:t}) = \upphi_d(y_{t-1} | \meanq_{t-1}(y_{t}, a_{1:t}), \Sigmaq_{t-1}(y_{t}, a_{1:t})) \eqsp.
\end{equation}

\subsection{Variational inference objective}

We consider the following loss function:
\begin{align}
\label{eq:loss_function}
     \lossDLPMD(\theta) &:=  \mE \left[\sum_{t=2}^{T} \left(\lossDLPMD_{t-1}(\theta, A_{1:T})\right)^{r}\right], \qquad \text{where} \\
    \lossDLPMD_{t-1}(\theta, A_{1:t}) &:= \mE \left[ \KL\left( \kArg{t-1 | t,0,a} (\cdot | Y_t, Y_0,  A_{1:t}) \ \| \ \qArg{t-1|t,a} (\cdot | Y_t,  A_{1:t})\right) \Bigl|  A_{1:t} \right] \eqsp,
\end{align}
and $r>0$, $\kArg{t-1|0,t, A}$ denotes the conditional density of $Y_{t-1}$ given $Y_0, Y_{t}$ and $A_{1:T}$. In order to ensure that the expectations with respect to $A_{1:T}$ are finite, we need to choose $r < \frac{\alpha}{2}$ when $\alpha \in (1,2)$. For simplicity, in the rest of the paper, we will use $r=\frac{1}{2}$. 

A comparison with DDPM's loss function $\lossDDPM$ (cf.\ \eqref{eqn:loss_ddpm}) immediately illustrates that, thanks to Theorem~\ref{thm:main_stable}, $\lossDLPMD$ is almost identical to $\lossDDPM$ up to taking expectations with respect to one-dimensional random variables and the taking square-root of the summands\footnote{We note that, with the choices of $\alpha=2$ and $r=1$, we exactly recover DDPM.}.
We show in \Cref{app:principled_approach} that, alike DDPM, our loss is obtained from a $\KL$ minimization principle, serving as an upper bound to:
\begin{equation}
\mE \left[ \left(\KL(\kArg{0|a}(\cdot |A_{1:T}) | \qArg{0|a}(\cdot |A_{1:T}))\right)^{r}\right]\eqsp,
\end{equation}
when $r \in (0,1]$.
This additionally shows how any zero of the loss corresponds to a perfect generative model, while maintaining a similar objective function. The crucial property of \eqref{eq:loss_function} is that, since both $\kArg{t-1 | t,0,a}$ and $\qArg{t-1|t,a}$ are Gaussian (thanks to the conditioning), the $\KL$ term admits a closed-form analytical formula, as in the case of DDPM. 



From a practical perspective, \eqref{eq:loss_function} suggests that we can use the same software architecture as for DDPM, with a slight modification to compute the outer expectation, which can be simply estimated by a Monte Carlo, or median-of-means procedure (\cite{lugosi2019mean}). In order to obtain a final denoising training loss, we provide three design choices, for which the full details are given in \Cref{app:dlpm_training_loss}. They are similar to what is classically done in diffusion models (\cite{ho2020denoising, nichol2021improved, karras2022elucidating}):

\hyperlink{D1}{D1.} We set a fixed variance $\Sigmaq_{t} = \Sigmak_t$ for the reverse process.

\hyperlink{D2}{D2.} We reparameterize the model to predict the value of $\epsilon_t(y_t, y_0)$ with a network $\epsilonq_t$, setting
\begin{equation} \label{app:eq:m_to_eps}
    \meanq_{t-1}(Y_t, A_{1:t}) = \frac{1}{\gstep[t]}\left(Y_{t} - \sall[t] \GammaA \epsilonq_{t}(Y_t, A_{1:t})\right)\eqsp.
\end{equation}
Moreover, we drop the dependency of $\epsilonq_t$ on $\listTA$, making $\epsilonq_t$ only a function of $Y_t$. This enables re-using classical diffusion models network architectures.

\hyperlink{D3}{D3.} Assuming \hyperlink{D1}{D1}, \hyperlink{D2}{D2}, we obtain $\lossDLPMD_{t-1}(\theta)
    =  \mE \left[ \lambda_{t, A_{1:t}}^{2} \| \epsilonq_{t}(Y_t, A_{1:t}) -\epsilon_t(Y_t, Y_0) \|^{2} \right]$
where $\lambda_{t, a_{1:t}} = {\Gammaa \sall[t]}/{2\gstep[t]\Sigmak_t}$, and $\epsilon_{t}(y_t, y_0)= {(y_{t} - \gall y_0)} / {\sall}$.
We then fix $\lambda_{t, a_{1:t}} =1$.

\begin{proposition}[Simplified denoising loss]
\label{prop:denoising}
With the design choices \hyperlink{D1}{D1}, \hyperlink{D2}{D2}, \hyperlink{D3}{D3}, we obtain the simplified denoising objective function:
\begin{equation}
\label{eq:simple_loss}
 \lossDLPMSimple(\theta) = \sum_{t=1}^{T}\mE \left[ \mE \left( \| \epsilonq_t(Y_{t}) - \epsilon_{t}(Y_t, Y_0) \|^{2} \ \big| \ A_{1:t} \right)^{1/2}\right]\eqsp,
\end{equation}
where the model $\epsilonq_t$ is designed to fit the noise $\epsilon_{t}(Y_t, Y_0) = {(Y_{t} - \gall Y_0)} / {\sall}$ added at time-step $t$, and $\gall, \sall[t]$ are as given in \Cref{app:gall_sall}. Thus the model is not learning a heavy-tailed distribution directly, but the light-tailed conditional distribution. 
\end{proposition}

See \Cref{app:dlpm:loss:intuitive} for the derivations required for \Cref{prop:denoising}. 
Finally, in \Cref{app:less_rv}, we show that on some conditions, satisfied under design choices \hyperlink{D1}{D1}, \hyperlink{D2}{D2}, \hyperlink{D3}{D3}, the expectation of each term in $\lossDLPMD$ can be rewritten as an expectation with respect to only one univariate random variable (as opposed to $t$ variables, i.e., $A_{1:t}$), reducing the additional computational burden of accommodating heavy tails. As we will estimate the expectations by Monte Carlo averaging, reducing the number of random variables in the expectation is equally important to reduce the error in the estimation. The resulting loss is given in \Cref{app:prop_training_loss_simplified_less}.

\subsection{Deterministic sampling with DLIM}

Using the same techniques as in denoising diffusion implicit models (DDIM, \cite{DDIM}), we can recover a deterministic sampling scheme. We call this algorithmic extension Denoising L\'{e}vy Implicit Models (DLIM), which details are given in \Cref{app:dlim}. We recapitulate the resulting sampling and training algorithms in \Cref{app:sec_algo}. In this context and in the case of the non-isotropic Cauchy distribution $(\alpha = 1)$, it is possible to bypass the data augmentation technique, since a closed-form $\KL$ divergence between two such distributions exists. Full details about these Cauchy DLIM are given in \Cref{app:cauchy_dlim}, but we leave their experimental exploration to future work. 


\subsection{Comparing DLPM to LIM} 
The objective function \eqref{eq:loss_function} is slightly different from the one obtained in the continuous setting of LIM by \cite{NEURIPS2023_score_based_levy_lim}. The training equations are very similar, and can be reformulated to involve a denoising loss (see \Cref{app:comparing_lim_dlpm}, and \eqref{app:eq:lim_vs_dlpm}):
\begin{equation}
    \mathcal{L}_{t-1} : \theta \mapsto \mE \left(\| \epsilonq_t(X_t) - \epsilon_t(X_t, X_0) \|_p^{\eta} \right).
\end{equation}
As a refresher, in the case of DDPM, one sets $p = 2, \eta = 2$.
In the case of DLPM, our discussion leads us to the choice $p = 2, \eta = 1$ (see \eqref{eq:simple_loss}). In the case of LIM, the theory relies on a squared $\ell_2$ loss, setting $p = 2, \eta = 2$ in order to properly derive the loss and effectively approximate the true score of the data, at various noisescales. One must therefore make the assumption that $\mathcal{L}_{t-1}$ is not infinite for each parameter $\theta$ considered, which may not hold since $\epsilon_t(X_t, X_0)$ is $\alpha$-stable distributed and admits no variance. In the experiments for LIM, one is forced to revert to an $\ell_1$ loss, by setting $\eta=1, p = 1$, to obtain a stable training, potentially indicating that the loss is infinite.

Additionally, DLPM and LIM yield different backward processes, which in turn lead to distinct sampling algorithms -- cf.\ Table~\ref{tab:samplers} in the Appendix. Finally, the LIM framework can only accommodate isotropic noise. We refer to \Cref{app:comparing_lim_dlpm} for a detailed comparison between DLPM and LIM.

\section{Experiments}
After the groundwork of the previous sections, we design experiments to demonstrate the practical strengths of our DLPM approach as compared to LIM, apart from its technical simplicity. We recall that setting $\alpha=2$ simply reverts LIM to classical diffusion, and DLPM reverts to DDPM, appart from the square-root in the loss function.
As specified in (\cite{NEURIPS2023_score_based_levy_lim}, Appendix G), LIM relies on gradient and noise clipping, which introduces extra hyperparameters that must be fine-tuned for each dataset. In the experiments, we use these clipping parameters only when specified.
The experimental details relative to this section are available in Appendix~\ref{app:experiment}. 

As our loss function \eqref{eq:loss_function} involves an expectation with respect to $A_{1:T}$, we propose estimating it by using the \emph{median-of-means} estimator, which is known to have better performance for heavy-tailed distributions (\cite{lugosi2019mean}). For an integer $M$, this approach requires sampling $M^2$ many $A_{1:T}$ terms, then split them into $M$ groups of size $M$. To approximate the expectation, we take the sample mean of each group, and finally take the median of the computed $M$ sample means.
In our experiments, we explore $M=1$  (approximating the expectation with only one sample), denoted simply by DLPM, and $M=5$, denoted by \dlpmm{}. Similarly, we denote DLIM and DLIM$_{5}$ for the corresponding deterministic sampling schemes. 

Finally, we consider the range $1.5 \leq \alpha \leq 2.0$. 
In our experiments on images, we make use of the dataset CIFAR10$_-$LT (long tail), that has been introduced in \cite{NEURIPS2023_score_based_levy_lim} as an unbalanced modification of the CIFAR10 dataset.

\subsection{Data coverage and mode collapse in two-dimensional data}
Before progressing to higher dimensional problems, we start with easily controlled and visualized two-dimensional datasets, in order to validate the competitiveness of our method in the contexts where heavy-tailed diffusions are of interest. In particular, we consider heavy-tailed and unbalanced multi-modal datasets. See \Cref{app:exp:2d} for details about the experimental setup in these settings.

\begin{wrapfigure}{r}{0.4\textwidth}
    \centering
    \caption{DLPM with $\alpha=1.7$ and $\alpha=2.$. The lighter-tailed process fails to capture the distribution's tail. }
        \begin{minipage}[l]{0.4\textwidth}
            \centering
            \begin{minipage}[r]{0.49\textwidth}
                \centering
            \includegraphics[width=0.99\linewidth]{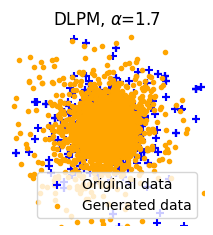}
            \end{minipage}
            \begin{minipage}[l]{0.49\textwidth}
                \centering
        \includegraphics[width=0.99\linewidth]{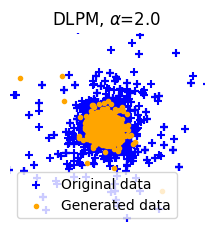}
            \end{minipage}
    \end{minipage}
  \label{fig:heavy_tailed_generation}
\end{wrapfigure}
\paragraph{Enhancing data coverage: capturing the tail of the distribution.} 
We set $d=2$ and generate data points distributed as $\stable_{\alpha}^{\text{i}}(0, 0.05\cdot\Idd)$,
with $\alpha=1.7$. Our aim is to test the ability of each method to cover the dataset correctly; the main challenge is to correctly capture the tails.
As we can visually observe in \Cref{fig:heavy_tailed_generation}, the backward diffusion process in the Gaussian case cannot produce truly heavy-tailed data, mainly stemming from the fact that its variance is always finite. As expected, we see improvements when using noise with $\alpha < 2$.

To quantify this behaviour, we utilize the mean square logarithmic error (MSLE) metric, which compares the relative error between the higher quantiles (\ie, the quantiles corresponding to the tail region that has probability $1-\xi$) of the true and the generated data distribution (see \Cref{app:metrics} for detailed definitions). We observe in \Cref{tab:msle} how, as $\alpha$ gets smaller, one gets better tail approximation. 
Furthermore, DLPM consistently outperforms LIM, indicating that the generation process benefits from the heavy-tailed denoising formulation, rather than the continuous-time one, in this setting.


\begin{figure}[ht]
\centering
\vspace{0pt}
\scalebox{0.85}{
\begin{tabular}{lcccccc}
\toprule
Method & {$\alpha=1.5$} & {$\alpha=1.6$} & $\alpha=1.7$ & $\alpha=1.8$ & $\alpha=1.9$ & $\alpha=2.0$ \\
\midrule
DLPM & {\textbf{0.160} ± 0.128} & {\textbf{0.081} ± 0.078} & \textbf{0.071} ± 0.028 & \textbf{0.099} ± 0.044 & \textbf{0.132} ± 0.101 & 0.798 ± 0.601 \\
{DDPM} & - & - & - & - & - & {0.528 ± 0.400} \\ 
           &    &    &    &    &    & { \hfill \footnotesize{\textit{1.0e-1}}} \\ 
LIM & {0.743 ± 0.290} & {0.497 ± 0.311} & 0.267 ± 0.077 & 0.653 ± 0.413 & 2.444 ± 1.067 & 1.239 ± 0.240 \\
    & { \hfill \footnotesize{\textit{1.0e-08}}} & { \hfill \footnotesize{\textit{8.6e-06}}} & { \hfill \footnotesize{\textit{1.3e-10}}} & { \hfill \footnotesize{\textit{8.8e-06}}} & { \hfill \footnotesize{\textit{7.9e-09}}} & { \hfill \footnotesize{\textit{5.0e-3}}} \\
\bottomrule
\end{tabular}
}
\captionof{table}{$\text{MSLE}_{\xi = 0.95} \downarrow$ averaged over 20 runs. Figures below scores corresponds to $p$-values from Welch's $t$-test (assuming unequal variances), comparing the mean of DLPM with the given method.}
\vspace{0pt}
\label{tab:msle}
\end{figure}


\begin{wrapfigure}{r}{0.2\textwidth}
\centering
\includegraphics[width=0.95\linewidth]{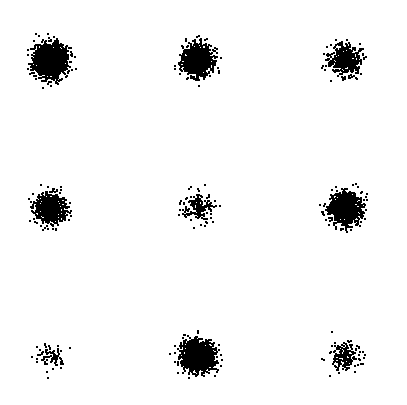}
\captionof{figure}{Gaussian grid}
\label{fig:gmm_grid}
\end{wrapfigure}
\paragraph{Enhancing data coverage: addressing mode collapse.}
To assess the robustness of DLPM to mode collapse, we consider an unbalanced mixture of nine Gaussian distributions. 
We set their standard deviation to $0.05$ and arrange them in a grid-like pattern with equal spacing, in the square $[-1, 1]^2$. 
Their mixture weights range from $.01$ to $.3$
\footnote{The exact mixture weights are $\{.01,.02, .02, .05, .1, .15, .15, .2, .3\}$.}.
We use the $\fpr$ score (\ie, the harmonic mean of precision and recall, see \Cref{app:metrics}) to assess, in a single summary statistic, the quality and diversity in the generated data. 
As shown in \Cref{tab:gmm_grid}, we are able to achieve improved scores by choosing a tail index $\alpha < 2$ with DLPM.
This is not necessarily the case for LIM, which is consistently outperformed. 
Finally, \dlpmm{} shows its strengths with better performance over all the range of $\alpha$, though at the cost of 5 times the run time.

\begin{figure}[ht]
\begin{minipage}{\textwidth}
        
\begin{minipage}[l]{\textwidth}
    \centering
    \scalebox{0.8}{
    \begin{tabular}{lcccccc}
        \toprule
        Method & {$\alpha=1.5$} & {$\alpha=1.6$} & $\alpha=1.7$ & $\alpha=1.8$ & $\alpha=1.9$ & $\alpha=2.0$ \\
        \midrule
        DLPM               & {0.933 ± 0.018}  & {0.923 ± 0.005}   & {0.933 ± 0.028}  & {0.923 ± 0.024}  & {0.907 ± 0.034}  & {0.862 ± 0.028} \\
        
        DLPM$_5$           & \textbf{0.944 ± 0.013} & \textbf{0.943 ± 0.021} & \textbf{0.943 ± 0.010} & \textbf{0.941 ± 0.014} & \textbf{0.928 ± 0.016} & - \\
                                   & \footnotesize{ \hfill \textit{9.0e-3}} & \footnotesize{ \hfill \textit{1.6e-05}} & \footnotesize{ \hfill \textit{7.4e-2}} & \footnotesize{ \hfill \textit{9.0e-4}} & \footnotesize{ \hfill \textit{3.9e-3}} & \\
                                   
        LIM                & {0.842 ± 0.039}  & {0.850 ± 0.046}  & {0.868 ± 0.034}  & {0.874 ± 0.030}  & {0.884 ± 0.017}  & {0.874 ± 0.027} \\
                           & \hfill \footnotesize{\textit{1.7e-14}} & \footnotesize{ \hfill \textit{1.3e-09}} & \footnotesize{ \hfill \textit{5.7e-11}} & \footnotesize{ \hfill \textit{3.9e-09}} & \footnotesize{ \hfill \textit{1.9e-3}} & \footnotesize{ \hfill \textit{9.6e-2}} \\
    DDPM               & -            & -            & -            & -            & -            & {0.867 ± 0.029} \\ 
                           &              &              &              &              &              & \footnotesize{ \hfill \textit{5.0e-1}} \\
        \bottomrule
    \end{tabular}
    }
    \captionof{table}{$\fpr \uparrow$ score, averaged over {30 runs}. Figures below scores corresponds to $p$-values from Welch's $t$-test (assuming unequal variances), comparing the mean of DLPM with the given method.}
    \label{tab:gmm_grid}
\end{minipage}
\end{minipage}
\end{figure}

\subsection{Experiments on image data}
To fairly illustrate the differences between LIM and DLPM, we use the same improved DDPM neural network architecture, as designed in \cite{nichol2021improved}. The specific configuration for each dataset is carefully described in \Cref{app:exp:image}.
Our experiments are designed to compare deterministic and stochastic generation methods under varying conditions.
As a visual check, examples of generated images are listed in \Cref{app:exp:additional}. 

\begin{figure}[ht]
\centering

    \begin{minipage}[t]{0.7\textwidth}
    \centering
    \captionof{table}{FID$\downarrow$, 1000 sampling steps for LIM and DLPM, 25 sampling steps for LIM-ODE and DLIM.}
    \scalebox{0.8}{
    \begin{tabular}{lccccc}
    MNIST & $\alpha=1.5$ & $\alpha=1.7$ & $\alpha=1.8$ & $\alpha=1.9$ & $\alpha=2.0$ \\
    \midrule  
                DDPM         & -       & -       & -       & -       &  \textbf{3.43}  \\
                  {LIM}          & {14.37}    & {11.54}    & {11.18}    & {13.75}    & 11.69 \\
                      \quad {\textit{w/ clipping}}   & \hfill {\textit{4.08}}    & \hfill {\textit{5.17}}    & \hfill {\textit{6.81}}    & \hfill {\textit{11.20}}   &  \\
                  \dlpmm       & \textbf{3.80}   & 3.03    & \textbf{2.51}   & \textbf{2.71}   & - \\
                  DLPM         & 5.39    & \textbf{2.94}   & 2.93    & 3.24    & 3.63 \vspace{2mm} \\
                  {DDIM} &   -    &   -    &    -   &   -    & \textbf{{5.16}} \\
                  LIM-ODE      & {49.63}   & {78.59}   & {92.93}   & {109.48}  & 29.04\\
                  \quad {\textit{w/ clipping}}  & \hfill {\textit{45.72}}   & \hfill {\textit{68.15}}   & \hfill {\textit{85.09}}   & \hfill {\textit{113.20}}  &  \\
                  \dlimm       & \textbf{3.37}   & 2.93    & 3.44    & 4.31    & - \\
                  DLIM         & 3.38    & \textbf{2.81}   & \textbf{3.18}   & \textbf{3.27}   & 5.18 \vspace{7pt}\\
    
     CIFAR10\_LT & & & & & \\
    \midrule  
    DDPM & - & - & - & - & \textbf{19.05}  \\
    {LIM}          & {75.38}   & {35.15}   & {31.14}   & {21.68}   & 21.56 \\
     \quad {\textit{w/ clipping}}       & \hfill {\textit{16.13}} & \hfill {\textit{16.21}} & \hfill {\textit{17.67}} & \hfill {\textit{19.24}} &  \\
         DLPM & \textbf{16.10} & \textbf{18.00} & \textbf{19.94} & \textbf{20.21} & 21.07 \vspace{2mm} \\
    {DDIM} &   -    &     -  &  -     &  -     & \textbf{{23.44}} \\
    {LIM-ODE}      & {42.07}   & {91.64}   & {105.95}   & {407.79}   & 32.00 \\ 
     \quad {\textit{w/ clipping}} & \hfill {\textit{30.17}} & \hfill {\textit{65.78}} & \hfill {\textit{84.55}} & \hfill {\textit{101.70}} &  \\
     DLIM & \textbf{20.69} & \textbf{20.77} & \textbf{21.96} & \textbf{22.79} & 23.99 \\
    \bottomrule
    \end{tabular}
    }
    \label{tab:image_data}
    \end{minipage}
    \hfill 
    \begin{minipage}[t]{0.28\textwidth}
    \centering
        \captionof{figure}{FID$\downarrow$, $\alpha=1.7$}
        \begin{minipage}[b]{1\textwidth}
        \centering
        \includegraphics[width=1\linewidth]{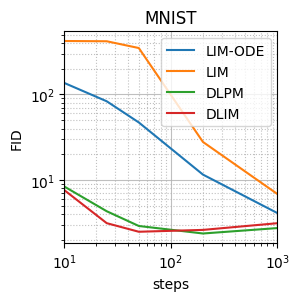}
        \end{minipage}
        \begin{minipage}[t]{1\textwidth}
        \centering
        \includegraphics[width=1\linewidth]{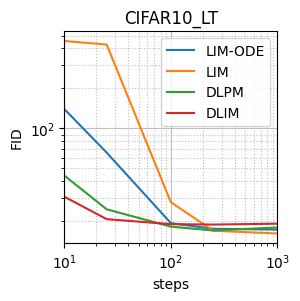}
        \end{minipage}
    \vspace{-25pt}
    \label{fig:lim_vs_dlpm_reverse_steps}
    \end{minipage}
\end{figure}

\paragraph{Convergence speed.}

Consistent with existing literature \cite{DDIM}, our findings as shown in \Cref{fig:lim_vs_dlpm_reverse_steps} confirm that deterministic generation outperforms its stochastic counterpart significantly, especially when fewer than 100 diffusion steps are used, on both MNIST and CIFAR10\_LT. As the number of diffusion steps increases, both of these sampling methods produce similar results. This observation highlights that the advantages of the diffusion process do not only stem from increased randomness at sampling time. These heavy-tailed processes may define more appropriate vector field  on which the noise is transported back to the original data distribution, which would lead to improved model performance (see \cite{karras2022elucidating} for similar discussions on DDPM vs DDIM).

The previous observations are quantitatively supported in \Cref{tab:image_data}, where we present results for both deterministic and stochastic sampling strategies. We compare both methods on stochastic generation at a high step count, to compare their performance at their best regime, and on deterministic generation at a small step count, to assess the tradeoff in computations/quality offered by both methods. As we can see, DLPM surpasses LIM on both datasets. Moreover, these results show that LIM’s performance deteriorates significantly when clipping is not used, raising questions about whether the framework of LIM is inherently well-suited for heavy-tailed distributions. 
More interestingly, we observe in \Cref{tab:image_data} that DLPM consistently outperforms LIM and offers satisfying image quality at low number of steps, both for stochastic and deterministic sampling.

Generated images after $25$ steps achieve a FID score of $2.81$ on MNIST and of $20.69$ on CIFAR10\_LT, as compared to respectively $45.72$ and $30.17$ for LIM-ODE with clipping. On MNIST, with $\alpha=1.7$, DLIM is able to match the sample quality of DLPM with 40 times less diffusion steps, further proving its efficacy.
To visualize these behaviours, we display on \Cref{fig:dlpm_vs_lim_small_steps} 
different generation with varying time horizon $T$. We can see how the backward process defined by DLIM is able to approach the true data distribution more accurately.
\begin{figure}[t]
\centering
\begin{minipage}[b]{0.16\linewidth}
    \includegraphics[width=\linewidth]{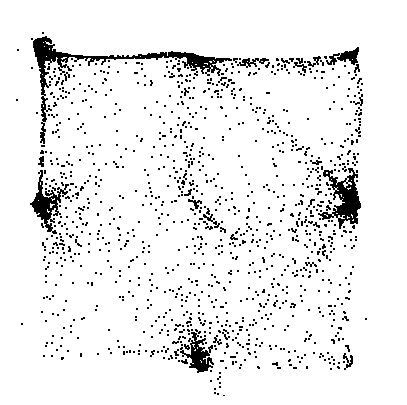}
    \caption*{\textbf{DLIM}, 5 steps}
\end{minipage}
\begin{minipage}[b]{0.16\linewidth}
    \includegraphics[width=\linewidth]{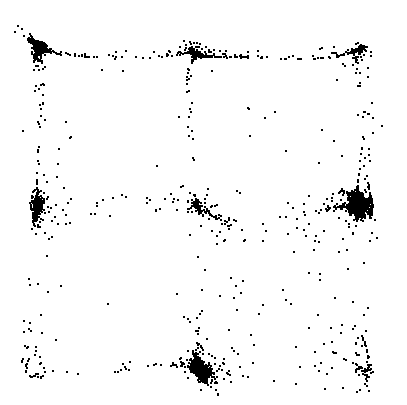}
    \caption*{10 steps}
\end{minipage}
\begin{minipage}[b]{0.16\linewidth}
    \includegraphics[width=\linewidth]{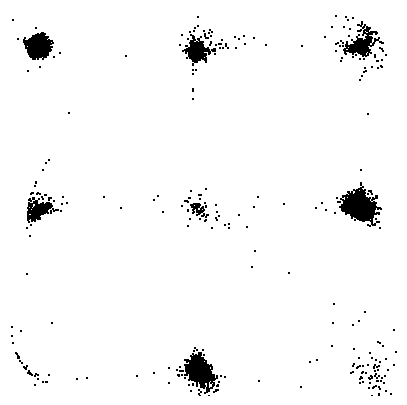}
    \caption*{25 steps}
\end{minipage}
\hfill
\begin{minipage}[b]{0.16\linewidth}
    \includegraphics[width=\linewidth]{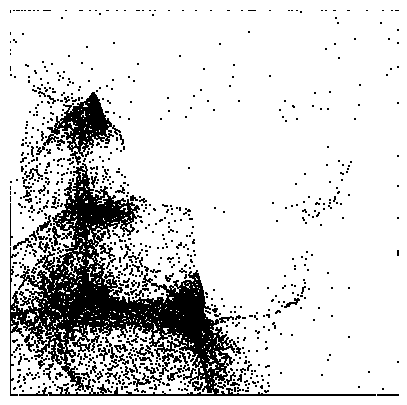}
    \caption*{\textbf{LIM}, 5 steps}
\end{minipage}
\begin{minipage}[b]{0.16\linewidth}
    \includegraphics[width=\linewidth]{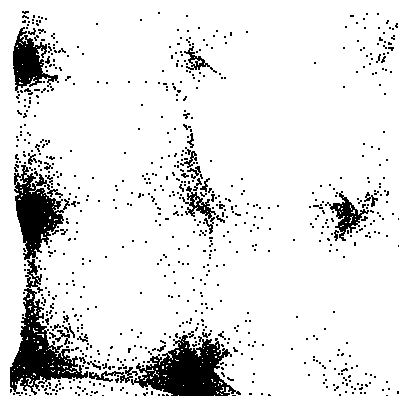}
    \caption*{10 steps}
\end{minipage}
\begin{minipage}[b]{0.16\linewidth}
    \includegraphics[width=\linewidth]{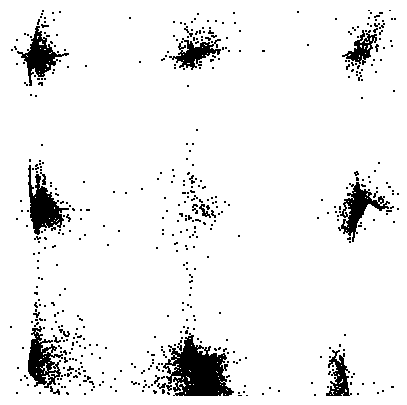}
    \caption*{25 steps}
\end{minipage}
\caption{DLIM and LIM-ODE with small number of steps, on the Gaussian grid of \Cref{fig:gmm_grid}.}
\label{fig:dlpm_vs_lim_small_steps}
\end{figure}

Eventhough lower $\alpha$ usually entails lower FID in this table, LIM-ODE shows worse performance than Gaussian diffusion at 25 reverse steps;
since the image quality is still monotonically decreasing with $\alpha$, except for $\alpha=2$, we can conjecture that the initial instability introduced by heavy-tails are quickly counterbalanced by their benefits.
\dlpmm{} shows consistent improvement over baseline, more particularly for stochastic sampling. 
We provide results for non-isotropic generation, and additional metrics on image data in \Cref{app:exp:additional}, further supporting our claims.

\section{Conclusion}

In this study, we proposed DLPM and DLIM, as heavy-tailed generalizations of DDPM and DDIM. Contrary to similar methods, we believe our approach will be more accessible to the community, thanks to its elementary tools. 
The various experiments conducted suggest DLPM is more effective in leveraging the characteristics of heavy-tailed distributions, providing robust performance across heavy-tailed data, unbalanced datasets and requiring a lower number of diffusion steps.

\subsubsection*{Acknowledgment}
{
A.D is funded by the European Union (ERC, Ocean, 101071601). D.S and U.S are partially funded by the European Union (ERC, Dynasty, 101039676). Views and opinions expressed are however those of the author(s) only and do not necessarily reflect those of the European Union or the European Research Council Executive Agency. Neither the European Union nor the granting authority can be held responsible for them.

U.S is additionally funded by the French government under management of Agence Nationale de la Recherche as part of the ``France 2030'' program, reference ANR-23-IACL-0008 (PR[AI]RIE-PSAI).


The authors are grateful to the CLEPS infrastructure from the Inria of Paris for providing resources and support.

}

\bibliography{reference}

\begin{thebibliography}{38}
\providecommand{\natexlab}[1]{#1}
\providecommand{\url}[1]{\texttt{#1}}
\expandafter\ifx\csname urlstyle\endcsname\relax
  \providecommand{\doi}[1]{doi: #1}\else
  \providecommand{\doi}{doi: \begingroup \urlstyle{rm}\Url}\fi

\bibitem[Allouche et~al.(2022)Allouche, Girard, and
  Gobet]{allouche:ev_gan_heavy_tail}
Micha{\"e}l Allouche, St{\'e}phane Girard, and Emmanuel Gobet.
\newblock {EV-GAN: Simulation of extreme events with ReLU neural networks}.
\newblock \emph{{Journal of Machine Learning Research}}, 23\penalty0
  (150):\penalty0 1--39, 2022.
\newblock URL \url{https://hal.science/hal-03250663}.

\bibitem[Borak et~al.(2005)Borak, Härdle, and Weron]{stable_borak}
Szymon Borak, Wolfgang~Karl Härdle, and Rafał Weron.
\newblock Stable distributions.
\newblock \emph{Statistical Tools for Finance and Insurance}, 04 2005.
\newblock \doi{10.1007/3-540-27395-6_1}.

\bibitem[Brooks et~al.(2011)Brooks, Gelman, Jones, and
  Meng]{brooks2011handbook}
Steve Brooks, Andrew Gelman, Galin Jones, and Xiao-Li Meng.
\newblock \emph{Handbook of Markov Chain Monte Carlo}.
\newblock CRC press, 2011.

\bibitem[Chen et~al.(2020)Chen, Zhang, Zen, Weiss, Norouzi, and
  Chan]{chen2020wavegradestimatinggradientswaveform}
Nanxin Chen, Yu~Zhang, Heiga Zen, Ron~J. Weiss, Mohammad Norouzi, and William
  Chan.
\newblock Wavegrad: Estimating gradients for waveform generation, 2020.
\newblock URL \url{https://arxiv.org/abs/2009.00713}.

\bibitem[Chyzak \& Nielsen(2019)Chyzak and
  Nielsen]{chyzak2019closedformformulakullbackleiblerdivergence}
Frédéric Chyzak and Frank Nielsen.
\newblock A closed-form formula for the kullback-leibler divergence between
  cauchy distributions, 2019.
\newblock URL \url{https://arxiv.org/abs/1905.10965}.

\bibitem[Deasy et~al.(2022)Deasy, Simidjievski, and Liò]{deasy2022heavytailed}
Jacob Deasy, Nikola Simidjievski, and Pietro Liò.
\newblock Heavy-tailed denoising score matching, 2022.

\bibitem[Dhariwal \& Nichol(2021{\natexlab{a}})Dhariwal and
  Nichol]{dhariwal2021diffusion}
Prafulla Dhariwal and Alex Nichol.
\newblock Diffusion models beat gans on image synthesis, 2021{\natexlab{a}}.

\bibitem[Dhariwal \& Nichol(2021{\natexlab{b}})Dhariwal and
  Nichol]{dhariwal2021diffusionmodelsbeatgans}
Prafulla Dhariwal and Alex Nichol.
\newblock Diffusion models beat gans on image synthesis, 2021{\natexlab{b}}.
\newblock URL \url{https://arxiv.org/abs/2105.05233}.

\bibitem[Gnedenko \& Kolmogorov(1968)Gnedenko and
  Kolmogorov]{gnedenko1968limit}
B.~V. Gnedenko and A.~N. Kolmogorov.
\newblock \emph{Limit distributions for sums of independent random variables}.
\newblock Translated from the Russian, annotated, and revised by K. L. Chung.
  With appendices by J. L. Doob and P. L. Hsu. Revised edition. Addison-Wesley
  Publishing Co., Reading, Mass.-London-Don Mills., Ont., 1968.

\bibitem[Ho et~al.(2020)Ho, Jain, and Abbeel]{ho2020denoising}
Jonathan Ho, Ajay Jain, and Pieter Abbeel.
\newblock Denoising diffusion probabilistic models, 2020.

\bibitem[Holt \& Nguyen(2023)Holt and Nguyen]{HoltGaussian}
William Holt and Duy Nguyen.
\newblock Essential aspects of bayesian data imputation.
\newblock \emph{SSRN}, 2023, jun 2023.
\newblock \doi{10.2139/ssrn.4494314}.
\newblock URL \url{http://dx.doi.org/10.2139/ssrn.4494314}.

\bibitem[Huang et~al.(2020{\natexlab{a}})Huang, Majka, and
  Wang]{huang2020approximationheavytaileddistributionsstabledriven}
Lu-Jing Huang, Mateusz~B. Majka, and Jian Wang.
\newblock Approximation of heavy-tailed distributions via stable-driven sdes,
  2020{\natexlab{a}}.
\newblock URL \url{https://arxiv.org/abs/2007.02212}.

\bibitem[Huang et~al.(2020{\natexlab{b}})Huang, Majka, and
  Wang]{huang2021approximation}
Lu-Jing Huang, Mateusz~B. Majka, and Jian Wang.
\newblock Approximation of heavy-tailed distributions via stable-driven sdes,
  2020{\natexlab{b}}.
\newblock URL \url{https://arxiv.org/abs/2007.02212}.

\bibitem[Karras et~al.(2022)Karras, Aittala, Aila, and
  Laine]{karras2022elucidating}
Tero Karras, Miika Aittala, Timo Aila, and Samuli Laine.
\newblock Elucidating the design space of diffusion-based generative models,
  2022.

\bibitem[Kingma \& Ba(2017)Kingma and Ba]{kingma2017adam}
Diederik~P. Kingma and Jimmy Ba.
\newblock Adam: A method for stochastic optimization, 2017.

\bibitem[Kong et~al.(2021)Kong, Ping, Huang, Zhao, and
  Catanzaro]{kong2021diffwaveversatilediffusionmodel}
Zhifeng Kong, Wei Ping, Jiaji Huang, Kexin Zhao, and Bryan Catanzaro.
\newblock Diffwave: A versatile diffusion model for audio synthesis, 2021.
\newblock URL \url{https://arxiv.org/abs/2009.09761}.

\bibitem[Lugosi \& Mendelson(2019)Lugosi and Mendelson]{lugosi2019mean}
G{\'a}bor Lugosi and Shahar Mendelson.
\newblock Mean estimation and regression under heavy-tailed distributions: A
  survey.
\newblock \emph{Foundations of Computational Mathematics}, 19\penalty0
  (5):\penalty0 1145--1190, 2019.

\bibitem[Nachmani et~al.(2021)Nachmani, Roman, and Wolf]{nachmani2021denoising}
Eliya Nachmani, Robin~San Roman, and Lior Wolf.
\newblock Denoising diffusion gamma models, 2021.

\bibitem[Nichol \& Dhariwal(2021)Nichol and Dhariwal]{nichol2021improved}
Alex Nichol and Prafulla Dhariwal.
\newblock Improved denoising diffusion probabilistic models, 2021.

\bibitem[Nobis et~al.(2024)Nobis, Springenberg, Aversa, Detzel, Daems,
  Murray-Smith, Nakajima, Lapuschkin, Ermon, Birdal, Opper, Knochenhauer, Oala,
  and Samek]{nobis2024generativefractionaldiffusionmodels}
Gabriel Nobis, Maximilian Springenberg, Marco Aversa, Michael Detzel, Rembert
  Daems, Roderick Murray-Smith, Shinichi Nakajima, Sebastian Lapuschkin,
  Stefano Ermon, Tolga Birdal, Manfred Opper, Christoph Knochenhauer, Luis
  Oala, and Wojciech Samek.
\newblock Generative fractional diffusion models, 2024.
\newblock URL \url{https://arxiv.org/abs/2310.17638}.

\bibitem[Nolan(2010)]{Nolan2010StableDistributions}
J.~P. Nolan.
\newblock \emph{Stable Distributions - Models for Heavy Tailed Data}.
\newblock {Birkh\"auser}, Boston, 2010.
\newblock In progress, Chapter 1 online at
  academic2.american.edu/$\sim$jpnolan.

\bibitem[Nolan(2020)]{nolan2020univariate}
John~P Nolan.
\newblock Univariate stable distributions.
\newblock \emph{Springer Series in Operations Research and Financial
  Engineering}, 10:\penalty0 978--3, 2020.

\bibitem[Ortigueira et~al.(2014)Ortigueira, Laleg-Kirati, and
  Machado]{Ortigueira_2014}
Manuel~D Ortigueira, Taous-Meriem Laleg-Kirati, and J~A~Tenreiro Machado.
\newblock Riesz potential versus fractional laplacian.
\newblock \emph{Journal of Statistical Mechanics: Theory and Experiment},
  2014\penalty0 (9):\penalty0 P09032, sep 2014.
\newblock \doi{10.1088/1742-5468/2014/09/P09032}.
\newblock URL \url{https://dx.doi.org/10.1088/1742-5468/2014/09/P09032}.

\bibitem[Sadat et~al.(2024)Sadat, Buhmann, Bradley, Hilliges, and
  Weber]{sadat2024cadsunleashingdiversitydiffusion}
Seyedmorteza Sadat, Jakob Buhmann, Derek Bradley, Otmar Hilliges, and Romann~M.
  Weber.
\newblock Cads: Unleashing the diversity of diffusion models through
  condition-annealed sampling, 2024.
\newblock URL \url{https://arxiv.org/abs/2310.17347}.

\bibitem[Sajjadi et~al.(2018)Sajjadi, Bachem, Lucic, Bousquet, and
  Gelly]{sajjadi2018assessing}
Mehdi S.~M. Sajjadi, Olivier Bachem, Mario Lucic, Olivier Bousquet, and Sylvain
  Gelly.
\newblock Assessing generative models via precision and recall, 2018.

\bibitem[Samorodnitsky et~al.(1996)Samorodnitsky, Taqqu, and
  Linde]{samorodnitsky1996stable}
Gennady Samorodnitsky, Murad~S Taqqu, and RW~Linde.
\newblock Stable non-gaussian random processes: stochastic models with infinite
  variance.
\newblock \emph{Bulletin of the London Mathematical Society}, 28\penalty0
  (134):\penalty0 554--555, 1996.

\bibitem[Simsekli(2017)]{pmlr-simsekli17a}
Umut Simsekli.
\newblock Fractional {L}angevin {M}onte carlo: Exploring {L}evy driven
  stochastic differential equations for {M}arkov chain {M}onte {C}arlo.
\newblock In Doina Precup and Yee~Whye Teh (eds.), \emph{Proceedings of the
  34th International Conference on Machine Learning}, volume~70 of
  \emph{Proceedings of Machine Learning Research}, pp.\  3200--3209. PMLR,
  06--11 Aug 2017.
\newblock URL \url{https://proceedings.mlr.press/v70/simsekli17a.html}.

\bibitem[Simsekli et~al.(2020)Simsekli, Zhu, Teh, and
  Gurbuzbalaban]{simsekli2020fractional}
Umut Simsekli, Lingjiong Zhu, Yee~Whye Teh, and Mert Gurbuzbalaban.
\newblock Fractional underdamped langevin dynamics: Retargeting sgd with
  momentum under heavy-tailed gradient noise.
\newblock In \emph{International conference on machine learning}, pp.\
  8970--8980. PMLR, 2020.

\bibitem[Sohl-Dickstein et~al.(2015)Sohl-Dickstein, Weiss, Maheswaranathan, and
  Ganguli]{sohldickstein2015deep}
Jascha Sohl-Dickstein, Eric~A. Weiss, Niru Maheswaranathan, and Surya Ganguli.
\newblock Deep unsupervised learning using nonequilibrium thermodynamics, 2015.

\bibitem[Song et~al.(2020)Song, Meng, and Ermon]{DDIM}
Jiaming Song, Chenlin Meng, and Stefano Ermon.
\newblock Denoising diffusion implicit models.
\newblock \emph{CoRR}, abs/2010.02502, 2020.
\newblock URL \url{https://arxiv.org/abs/2010.02502}.

\bibitem[Song \& Ermon(2020)Song and Ermon]{song2020generative}
Yang Song and Stefano Ermon.
\newblock Generative modeling by estimating gradients of the data distribution,
  2020.

\bibitem[Song et~al.(2021)Song, Sohl-Dickstein, Kingma, Kumar, Ermon, and
  Poole]{song2021scorebased}
Yang Song, Jascha Sohl-Dickstein, Diederik~P. Kingma, Abhishek Kumar, Stefano
  Ermon, and Ben Poole.
\newblock Score-based generative modeling through stochastic differential
  equations, 2021.

\bibitem[Song et~al.(2023)Song, Dhariwal, Chen, and
  Sutskever]{song2023consistencymodels}
Yang Song, Prafulla Dhariwal, Mark Chen, and Ilya Sutskever.
\newblock Consistency models, 2023.
\newblock URL \url{https://arxiv.org/abs/2303.01469}.

\bibitem[Vaswani et~al.(2023)Vaswani, Shazeer, Parmar, Uszkoreit, Jones, Gomez,
  Kaiser, and Polosukhin]{vaswani2023attention}
Ashish Vaswani, Noam Shazeer, Niki Parmar, Jakob Uszkoreit, Llion Jones,
  Aidan~N. Gomez, Lukasz Kaiser, and Illia Polosukhin.
\newblock Attention is all you need, 2023.

\bibitem[Vincent(2011)]{vincent_denoising_score_matching}
Pascal Vincent.
\newblock A connection between score matching and denoising autoencoders.
\newblock \emph{Neural Comput.}, 23\penalty0 (7):\penalty0 1661–1674, jul
  2011.
\newblock ISSN 0899-7667.
\newblock \doi{10.1162/NECO_a_00142}.
\newblock URL \url{https://doi.org/10.1162/NECO_a_00142}.

\bibitem[Yang et~al.(2024)Yang, Zhang, Song, Hong, Xu, Zhao, Zhang, Cui, and
  Yang]{yang2024diffusion}
Ling Yang, Zhilong Zhang, Yang Song, Shenda Hong, Runsheng Xu, Yue Zhao, Wentao
  Zhang, Bin Cui, and Ming-Hsuan Yang.
\newblock Diffusion models: A comprehensive survey of methods and applications,
  2024.

\bibitem[Yoon et~al.(2023)Yoon, Park, Kim, and
  Lim]{NEURIPS2023_score_based_levy_lim}
Eun~Bi Yoon, Keehun Park, Sungwoong Kim, and Sungbin Lim.
\newblock Score-based generative models with l\'{e}vy processes.
\newblock In A.~Oh, T.~Neumann, A.~Globerson, K.~Saenko, M.~Hardt, and
  S.~Levine (eds.), \emph{Advances in Neural Information Processing Systems},
  volume~36, pp.\  40694--40707. Curran Associates, Inc., 2023.
\newblock URL
  \url{https://proceedings.neurips.cc/paper_files/paper/2023/file/8011b23e1dc3f57e1b6211ccad498919-Paper-Conference.pdf}.

\bibitem[Zhang et~al.(2024)Zhang, Zheng, Yao, Wang, Zhou, Zhang, and
  Wang]{zhang2024longtailed}
Tianjiao Zhang, Huangjie Zheng, Jiangchao Yao, Xiangfeng Wang, Mingyuan Zhou,
  Ya~Zhang, and Yanfeng Wang.
\newblock Long-tailed diffusion models with oriented calibration.
\newblock In \emph{The Twelfth International Conference on Learning
  Representations}, 2024.
\newblock URL \url{https://openreview.net/forum?id=NW2s5XXwXU}.

\end{thebibliography}
\bibliographystyle{iclr/iclr2025_conference}

\newpage 
\appendix
\section*{Appendix}

The Appendix is organized as follows: 
\begin{itemize}
\item In \Cref{app:sec_algo}, we provide the pseudo-code for the training and sampling algorithms of DLPM and DLIM.
    \item In \Cref{app:alpha_stable_noise}, we characterize the stability property of the $\alpha$-stable distribution, and we give explicit formulas for the distribution of the sum of two $\alpha$-stable random variables.
    \item In \Cref{app:sec_dlpm}, we provide the detailed theory and derivations of the DLPM framework. Working on the process $\listT{X_t}$ defined in \eqref{eqn:dlpm_recall}, and its data augmentation counterpart $\listT{Y_t}$ defined in \eqref{eq:forward_t_with_a}, we start in \Cref{app:forward_process} by characterizing the distribution of $X_t$ given $X_0$ from a given noise schedule $\{(\gstep, \sstep)\}_{t=1}^T$, as this defines the characteristic location $\gall$ and scale $\sall$ of the process at time $t$. Likewise we characterize the distribution of $Y_t$ given $Y_0, \listTA$, as this defines the characteristic location $\gall$ and variance $\Sall$ of the augmented process at time $t$.
    
    In \Cref{app:backward_process}, we focus on the Gaussian trick exploited by the process $\listT{Y_t}$ defined in \eqref{eq:forward_t_with_a}. This leads us to the an explicit formula for the Gaussian density of the backward process conditioned on a sequence $\{A_{t}\}_{t=1}^T$ of $\alpha$-stable random variables. 
    
    In \Cref{app:dlpm_training_loss}, we further put this characterization into good use by obtaining a closed-form formula for the loss function . Follows a discussion on design choices for the model and the loss function at hand, which leads us to the simplified training loss \eqref{eq:simple_loss} corresponding to the denoising loss for $\alpha$-stable diffusion.

    In \Cref{app:principled_approach} we provide a more principled approach to derive the loss function at hand.

    Finally, in \Cref{app:less_rv}, we provide a faster sampling strategy for training DLPM, computing each loss term $\lossDLPMD_{t-1}$ using only two heavy-tailed random variables per datapoint, instead of $t$ random variables.

\item In \Cref{app:dlim}, we adapt the setting for deterministic sampling in classical denoising diffusion DDIM (\cite{DDIM}) to our $\alpha$-stable framework. We naturally call this extension DLIM. Alike DDIM, it is true that the same neural networks can be used for both the DLIM and DLPM generation procedures.

\item In \Cref{app:lim}, we compare in more details the two discrete and continuous frameworks DLPM and LIM, underlining how they offer two distinct loss functions, training and sampling algorithms.

\item In \Cref{app:proofs}, we give proofs relative to our technique for faster training, as introduced in \Cref{app:less_rv}.

\item In \Cref{app:experiment}, we provide additional experimental details.
    
\end{itemize}

\section{Algorithms for DLPM and DLIM}\label{app:sec_algo}
In this section, we explicitly provide the algorithms needed to train and sample from the DLPM and DLIM generative methods. 

\subsection{Training}
The same models can be shared between DLPM and DLIM, as underlined in \Cref{app:dlim:equiv_dlpm}.
We introduce the values $\sall, \gall$ determined by the noise schedule, as presented in \Cref{app:gall_sall}:
\begin{equation}
    \gall =  \prod_{i = 1}^t \gstep\eqsp, \qquad \sall = \left[\sum_{i = 1}^t \left(\frac{\gall[t]}{\gall[i]}\sstep[i]\right)^{\alpha}\right]^{1/\alpha}\eqsp.
\end{equation}
We define $c_A = \cos^{2/\alpha}(\uppi\alpha/4)$ as in \Cref{thm:main_stable}.
We make the design choices \hyperlink{D1}{D1}, \hyperlink{D2}{D2}, \hyperlink{D3}{D3} for our model, as described in \Cref{app:dlpm:loss:intuitive}. Finally, we use the method for faster sampling as described in \Cref{app:less_rv}, \Cref{app:prop_training_loss_simplified_less}. 
The resulting training algorithm is given in \Cref{algo:dlpm_train}.

\begin{algorithm}
\caption{Loss function for DLPM} \label{alg:dlpm}
\begin{algorithmic}[1]
\Require model $\{\epsilonq_t\}_{t=1}^T$, noise schedule $\{(\gstep, \sstep)\}_{t=1}^T$, data $Y_0$
\State Sample $t\sim \mathcal{U}[1, T]$ 
\State Sample $\bar A_t \sim \stableA$
\State Sample $G_t \sim \normal(0, \Idd)$
\State $Y_t \gets \gall Y_0 + \sall \bar A_t^{1/2} G_t$
\State $L_{t-1} \gets  \| \epsilonq_t(Y_t) - \bar A_t^{1/2} G_t\|_2$

\noindent \Return $L_{t-1}$
\end{algorithmic}
\caption{DLPM training - simplified loss}
\label{algo:dlpm_train}
\end{algorithm}

\subsection{Sampling}
Given the noise schedule and a sequence $\{A_t\}_{t=1}^T$ of $\alpha$-stable random variables we define:
\begin{equation}\label{app:eq:dlpm:sampling}
\SallA[t] = \sum_{k = 1}^t \left(\frac{\gall[t]}{\gall[k]}\sqrt{A_k}\sstep[k]\right)^{2} \eqsp, \quad   \GammaA = 1 - \frac{\gamma_t^2 \SallA[t-1]}{\SallA[t]} \eqsp,
\end{equation}
see \Cref{app:prop_backward} and \Cref{app:eq:gall_Sall} from \Cref{app:gall_Sall} for precise statements. 
We give our sampling algorithm for DLPM in \Cref{algo:dlpm_sampling}.
\begin{algorithm}
\caption{Stochastic sampling (DLPM)}
\begin{algorithmic}[1]
\Require model $\{\epsilonq_t\}_{t=1}^T$, noise schedule $\{(\gstep, \sstep)\}_{t=1}^T$
\State Sample $Y_T \sim \stable(0, \sall[T] \Idd)$
\State Sample $\listTA$ \iid, $A_t \sim \stableA$
\For{$t \gets T \textrm{ to } 1$}
    \State Compute $\Sall[t], \Sall[t-1], \Gamma_t$ as in \eqref{app:eq:dlpm:sampling}
    \State Sample $G_t \sim \normal(0, \Idd)$
    \State $\Sigmaq_{t-1} \gets \Gamma_t \Sall[t-1]$
    \State $Y_{t-1} \gets \dfrac{Y_t}{\gstep} - {\Gamma_t \sall}\epsilonq_t(Y_t) + \sqrt{\Sigmaq_{t-1}}G_t$
\EndFor

\noindent \Return $Y_0$
\end{algorithmic}
\label{algo:dlpm_sampling}
\end{algorithm}

For DLIM, one can potentially provide $Y_T$ as input, in order to use the support of the noise distribution as a latent space, alike what is done by \cite{DDIM}. We give our DLIM sampling algorithm in \Cref{algo:dlim_sampling}.

\begin{algorithm}
\caption{Deterministic sampling (DLIM)}
\begin{algorithmic}[1]
\Require model $\{\epsilonq_t\}_{t=1}^T$, noise schedule $\{(\gstep, \sstep)\}_{t=1}^T$
\State Sample $Y_T \sim \stablesym$
\For{$t \gets T \textrm{ to } 1$}
    \State $Y_{t-1} \gets \dfrac{Y_t}{\gstep} -\left( \dfrac{\sall}{\gstep} - \sall[t-1] \right)\epsilonq_t(Y_t)$
\EndFor

\noindent \Return $Y_0$
\end{algorithmic}
\label{algo:dlim_sampling}
\end{algorithm}

\section{Additional Remark on $\alpha$-stable Distributions}
\label{app:alpha_stable_noise}

Stable distributions are closed under convolution for a fixed value of $\alpha$. Since convolution is equivalent to multiplication of the Fourier-transformed function, it follows that the product of two stable characteristic functions with the same $\alpha$ will yield another such characteristic function. We precisely characterize this stability property in the following proposition:

\begin{proposition}
[See {\cite[Proposition 1.3]{nolan2020univariate}}]
\label{app:prop:stab_alpha}
    Let $X_1, X_2$ be two random variables respectively distributed as $X_1 \sim \stable_{\alpha, \beta_1}(\mu_1, \sigma_1)$ and $X_2 \sim  \stable_{\alpha, \beta_2}(\mu_2, \sigma_2)$, with $\mu_1, \mu_2, \beta_1, \beta_2 \in \mR$ and $\sigma_1, \sigma_2 > 0$. Then, $X = X_1 + X_2$ is distributed as $\stable_{\alpha, \beta}(\mu, \sigma)$ where:
    \begin{equation}
            \mu = \mu_1 + \mu_2\eqsp, \qquad 
        \sigma = (\sigma_1^{\alpha} + \sigma_2^{\alpha})^{1/\alpha} \eqsp, \qquad 
        \beta = \dfrac{\beta_1 \sigma_1^{\alpha} + \beta_2 \sigma_2^{\alpha}}{\sigma_1^{\alpha} + \sigma_2^{\alpha}}\eqsp.
    \end{equation}
\end{proposition}
In particular, when $X_1, X_2$ are such that $X_1 \sim \stable_{\alpha}(0, \sigma_1), X_2 \sim \stable_{\alpha}(0, \sigma_2)$, then $X = X_1 + X_2 \sim \stable_{\alpha}(0, (\sigma_1^{\alpha} + \sigma_2^{\alpha})^{1/\alpha})$, which is the key relation used in the later characterizations of the distribution of our forward process.


\section{Theoretical derivations for DLPM}

In this section, we provide the detailed theory and associated derivations of the DLPM framework.

\label{app:sec_dlpm}

\subsection{Setting and notations}

We will denote by $\upphi_d(\cdot | \mu, \Sigma)$ the density of $\mathcal{N}(\mu, \Sigma)$, where $\mu \in \mR^d$ and $\Sigma \in \mR^{d\times d}$. We will denote by $\uppsia$ the density of $\stableA$, where $c_A = \cos^{2/\alpha}(\uppi\alpha/4)$.

\paragraph{Forward process.}\ We reintroduce the setting presented in \Cref{sec:dlpm}, with the noising schedule being denoted by $\{(\gstep, \sstep)\}_{t=1}^T$, and the following \textbf{forward process} on which DLPM is based:
\begin{equation}
\label{app:eq:dlpm:forward}
   X_0 \sim \pstar \eqsp, \quad \text{ and } \quad X_{t} = \gstep X_{t-1} + \sstep \epsilon_t^{(\alpha)},
\end{equation}
where $\pstar$ is the data distribution and $\{\epsa_t\}_{t=1}^T$ are independent random variables distributed as $\stablesym$.

\paragraph{Data augmentation process.}\ We also introduce the associated \textbf{data augmentation process}: 
\begin{equation}
\label{app:eq:dlpm:forward:augmented}
Y_0 \sim \pstar \eqsp, \quad \text{ and } \quad Y_{t}  = \gamma_t Y_{t-1} + \sigma_t A_{t}^{1/2} G_{t}  \eqsp,
\end{equation}
where $\listTG$ and $\listTA$ are independent random variables distributed according to $G_t \sim \gauss(0,\Idd)$ and $A_t \sim \stable_{\alpha/2,1}(0, c_A )$, with $c_A =  \cos^{2/\alpha}(\uppi\alpha/4)$.
From Theorem~\ref{thm:main_stable}, $\{Y_t\}_{t=0}^T$ is a Markov chain that admits the same distribution as $\{X_t\}_{t=0}^T$.
We will denote by $\pt$ the distribution of $Y_t$, and by $\kArg{t|t-1}(\cdot | \cdot)$ the transition density associated to the Markov chain \eqref{app:eq:dlpm:forward}. 

\paragraph{Backward process.}\ A \textbf{backward process} associated to the Markov chain $\listT{Y_t}$ is a Markov chain $\listT{\overleftarrow{Y'}_t}$ such that the two processes $\listT{\overleftarrow{Y'}_{T-t}}$ and $\listT{Y_t}$ have the same distribution. For ease of presentation and following classical notations, we will rather consider $\listT{\iY_{t}}$ where $\iY_t = \overleftarrow{Y'}_{T-t}$. We will denote by $ \ikArg{t-1|t}(\cdot |\cdot)$ the transition densities associated to the process $\listT{\iY_{t}}$.
Since the true backward process is never available to us, we will focus on an approximation induced by a variational family. We consider the process $\listT{\iYArg_t}$ where $\iYArg_T$ is distributed as $\stable(0, \sall \Idd)$, and the density of the distribution of $\iYArg_{t-1}$ conditioned on $\iYArg_{t}$ are given by $\qArg{t-1|t}(\cdot | \cdot)$, where $\theta \in \mR^D$ parameterizes a neural network. We also denote by $\qArg{0:T}$ the joint distribution of $\listT{\iYArg}$ and by $\qArg{t}$ the marginal distribution of $\iYArg_t$.

\paragraph{Further notations.}\ Finally, we denote by $p_t^{\alpha}(\cdot | a_{1:t})$, $\kArg{t|t-1, a}(\cdot | \cdot, a_{1:t})$, and $\qArg{t-1|t, a}(\cdot | \cdot, a_{1:t})$ the densities/transition densities associated to the processes $\listT{Y_t}, \listT{\iYArg_t}$ given $A_t = a_t$ for $1\leq t \leq T$. 
We we will also write $p_t^{\alpha}(\cdot | y_0, a_{1:t})$, $\kArg{t|t-1, 0, a}(\cdot | \cdot, y_0, a_{1:t})$, and $\qArg{t-1|t, 0, a}(\cdot | \cdot, y_0, a_{1:t})$ when further conditioning on $Y_0$.

\subsection{Forward process} \label{app:forward_process}
Let us now characterize the distribution of $X_t$ given $X_0$, and $Y_t$ given $Y_0, \listTA$, which are tractable thanks to \Cref{app:prop:stab_alpha}. These will come in handy, for instance when working on the backward process in \Cref{app:backward_process}.

\begin{proposition}[Distribution of $X_t$ given $X_0$]
\label{app:gall_sall}
    Let $\{X_t\}_{t=0}^T$ be the forward process as given in \eqref{app:eq:dlpm:forward}, and $\{(\gstep, \sstep)\}_{t=1}^T$ the noise schedule. Then the distribution of $X_t$ given $X_0$ is given for any $t$ by 
    \begin{equation} 
        X_t \eqd \gall X_0 + \sall \bar \epsilon_t 
    \end{equation}
    where $\bar \epsilon_t \sim \stablesym$, and $\gall, \sall$ are given by:
    \begin{equation}\label{app:eq:gall_sall}
        \gall =  \prod_{k = 1}^t \gstep[k] \eqsp, \qquad \sall = \left(\sum_{k = 1}^t \left(\dfrac{\gall[t]}{\gall[k]}\sstep[k]\right)^{\alpha}\right)^{1/\alpha} \eqsp.
    \end{equation}
\end{proposition}
The proof is an elementary induction based on \Cref{app:prop:stab_alpha}.
    

\begin{proposition}[Distribution of $Y_t$ given $Y_0, \listTA$] \label{app:gall_Sall}

Let $\{Y_t\}_{t=0}^T$ be the forward process as given in \eqref{app:eq:dlpm:forward:augmented}, $\{(\gstep, \sstep)\}_{t=1}^T$ the noise schedule, and $\listTA$ the associated $\alpha/2$-stable random variables, parameterizing the variance of the Gaussian noise increments. Then the distribution of $Y_t$ given $Y_0, \listTA$ is the Gaussian distribution with mean $\gall Y_0$  and covariance matrix $\Sall \Idd$, \ie, 
\begin{equation}
    Y_t \eqd \gall Y_0 + \SallA^{1/2} \bar G_t\eqsp,
\end{equation}
where $\bar G_t \sim \normal(0, \Idd)$, and
\begin{equation} \label{app:eq:gall_Sall}
    \gall[t] = \prod_{k = 1}^t \gstep[k]\eqsp, \qquad \Salla[t] = \sum_{k = 1}^t \left(\frac{\gall[t]}{\gall[k]}\sqrt{a_k}\sstep[k]\right)^{2}\eqsp.
\end{equation}
\end{proposition}
The proof is elementary and omitted. It is worth mentioning the following recurrence, to speedup the computation of the sequence $\{\Salla\}_{t=1}^T$:
\begin{equation}
    \Salla[t] = \sstep[t]^2 a_t + \gstep[t]^2 \Salla[t-1]\eqsp.
\end{equation}



\subsection{Backward process} \label{app:backward_process}
Consider the setting of the data augmentation approach as given in \eqref{app:eq:dlpm:forward:augmented}. By the same arguments used in \Cref{sec:dlpm}, it can be verified that a process starting from $p_T^{(\alpha)}$ and with transition densities $\ikArg{t-1|t}(y_{t-1} |y_t) \propto \pt(y_{t-1}) \kArg{t|t-1}(y_{t}| y_{t-1})$ for any $y_{t-1}, y_{t}$ is a backward process associated with $\listT{Y_t}$.
However, it raises two main problems. First (i), we cannot characterize the distribution of $\ikArg{t-1|t}(y_{t-1} |y_t)$, since we do not know the data distribution. Second (ii), in the case where $\alpha < 2$, we do not have access to an explicit expression for $\ikArg{t-1|t, 0}(y_{t-1} |y_t, y_0)$.

Regarding (i), we have access to the distribution of $Y_t$ given $Y_0$, so a valid strategy consists in devising a method relying on characterizing the backward of the process $\listT{Y_t}$ given $Y_0$. This is the classical strategy used in DDPM (\cite{ho2020denoising}), which is possible in the case $\alpha = 2$ since $\ikArg{t-1|t, 0}(y_{t-1} |y_t, y_0)$ admits an analytical expression for any $y_0, y_{t-1}, y_t$, thanks to the properties of the Gaussian distribution. 

Regarding (ii), in the case where $\alpha < 2$, we make use of the trick introduced in \Cref{thm:main_stable}, justifying the data augmentation approach. We will rather characterize the density of the Markov kernels associated to the backward of the process $\listT{Y_t}$ given $Y_0$ and $\listTA$. This time, since we manage Gaussian noise increments, we can fall back to the classical strategy, as we further develop in the following proposition.

\begin{proposition}[Density of the backward process associated to $\listT{Y_t}$ given $Y_0, \listTA$]
\label{app:prop_backward}

Consider the setting of the data augmentation approach as given in \eqref{app:eq:dlpm:forward:augmented}. Let $\{(\gamma_t, \sigma_t)\}_{t = 1}^T$ be the noise schedule at hand. Let $\kArg{t-1|0,t, a}(\cdot|y_{t}, y_0, a_{1:t})$ be the density of the backward process associated to $\listT{Y_t}$ given $Y_0$ and $\listTA$.
Then $\kArg{t-1|t,0,a}(\cdot|y_{t}, y_0, a_{1:t})$ is the density of a Gaussian distribution $\normal(\meank_{t-1}, \Sigmak_{t-1})$ with mean $\meank_{t-1}$ and variance $\Sigmak_{t-1}$ such that
\begin{equation} \label{app:eq:meank_Sigmak}
    \meank_{t-1}(y_t, y_0, a_{1:t}) = \frac{1}{\gstep}\left( y_{t} - \Gammaa \sall \epsilon_{t}(y_t, y_0) \right)\eqsp, \quad 
    \Sigmak_{t-1}(a_{1:t}) = \Gammaa \Salla[t-1]\eqsp,
\end{equation}
where 
\begin{align}
\Salla[t] &= \sum_{k = 1}^t \left(\frac{\gall[t]}{\gall[k]}\sqrt{a_k}\sstep[k]\right)^{2} \\
\epsilon_{t}(y_t, y_0) &= \frac{y_{t} - \gall y_0}{\sall} \\
\Gammaa &= 1 - \frac{\gstep^2 \Salla[t-1]}{\Salla[t]}\eqsp.    
\end{align}
Eventhough $\Gamma_t$ involves multiple heavy-tailed random variables, it is nonetheless bounded: $0 \leq \Gamma_t \leq 1$.
\end{proposition}

\begin{proof}
To determine $\kArg{t-1|t,0,a}(\cdot|y_{t}, y_0, a_{1:t})$, we need to work on the joint distribution of $(Y_{t-1}, Y_{t})$ conditioned on $Y_0, \listTA$, which is a Gaussian vector for which classical techniques will let us derive the distribution of $Y_{t-1}$ given $Y_{t}$. Before doing so we need to compute $\rho_t$ the covariance of $Y_{t-1}$ and $Y_{t}$ given $Y_0, \listTA$, which we do thanks to \Cref{app:gall_Sall}: 
\begin{equation}
    \rho_t = \text{Cov}(Y_{t}, Y_{t-1} | Y_0, A_{1:T}) = \gstep[t] \text{Cov}(Y_{t-1}, Y_{t-1} | Y_0, A_{1:T}) = \gstep[t] \Sall[t-1] \Idd.
\end{equation}
Denote by $\kArg{t-1, t | 0, a}$ the density of $(Y_{t-1}, Y_{t})$ conditioned on $Y_0, A_{1:T}$. Denote by $\upphi_d(\cdot| \mu, \Sigma)$ the density of a $d$-dimensional Gaussian distribution with mean $\mu$ and covariance $\Sigma$. From the results of \Cref{app:gall_Sall}, we can write
\begin{equation}
    \kArg{t-1, t | 0, a}(y_{t-1}, y_{t} | y_0, a_{1:t}) = \upphi_d\left(
\begin{pmatrix}
   y_{t-1} \\ y_{t}
\end{pmatrix}
\bigg |
\begin{pmatrix}
    \gall[t-1]y_0 \\ \gall[t]y_0
\end{pmatrix},
\begin{pmatrix}
    \Salla[t-1] \Idd & \rho_t \Idd \\ \rho_t \Idd & \Salla[t] \Idd
\end{pmatrix}
    \right)
\end{equation}
Then the distribution of $Y_{t-1}$ given $Y_{t}, Y_{0}, A_{1:T}$ is a Gaussian distribution $\normal(\meank_{t-1}, \Sigmak_{t-1})$ \cite[Theorem 3]{HoltGaussian} with mean $\meank_{t-1}$ and variance $\Sigmak_{t-1}$ satisfying:
\begin{align}
    \meank_{t-1}(y_t, y_0, a_{1:t}) &= \gall[t-1]y_0 + \dfrac{\rho_t}{\Salla[t]}\left(y_{t}  - \gall[t] y_0\right) \\   
    \Sigmak_{t-1}(a_{1:t}) &= \Salla[t-1] - \dfrac{\rho_t^2}{\Salla[t]}\eqsp.
\end{align}
By defining
\begin{equation}
    \epsilon_{t}(y_t, y_0) = \frac{y_{t} - \gall[t]y_0}{\sall[t]}\eqsp, \quad \Gammaa = 1 - \frac{\gstep[t]^2 \Salla[t-1]}{\Salla[t]}\eqsp,    
\end{equation}
we give the final expression for the mean $\meank_{t-1}$ and variance $\Sigmak_{t-1}$ of the distribution of $Y_{t-1}$ given $Y_{t}, Y_0$ and $\listTA$:
\begin{align}
\label{app:backward_formula}
    \meank_{t-1}(y_t, y_0, a_{1:t}) &= \dfrac{1}{\gstep[t]}\left(y_{t} - \sall[t] \Gammaa \epsilon_{t}(y_t, y_0)\right)\\
        \Sigmak_{t-1}(a_{1:t}) &= \GammaA \Salla[t-1]\eqsp.
\end{align}
Since $\Gammaa = 1 - \gstep^2 \Salla[t-1]/\Salla[t] = a_t\sstep^2/\Salla$ and $a_t, \gstep, \Sall, \sall > 0$, we have $0\leq \Gamma_t \leq 1$.

\end{proof}

\paragraph{Case $\alpha = 2$} As we set $\alpha=2$, the random variables $\listTA$ become deterministic, equal to $2$. One can check that in this case, with the variance preserving schedule 
\begin{equation}
    \gstep[t] = \sqrt{1 - \beta_t}\eqsp, \quad \gall[t]=\sqrt{\alpha_t}\eqsp, \quad \sstep[t] = \sqrt{\beta_t}\eqsp, \quad \sall[t] = \sqrt{1 - \alpha_t}\eqsp,
\end{equation} then:
\begin{equation}
    \Sall = 2 \sall^2\eqsp = 2 (1 - \alpha_t)\eqsp.
\end{equation}
Further noticing that $\gamma_t = \gall / \gall[t-1]$, one computes
\begin{align}
    \Gamma_t &= 1 - \frac{\sall[t-1]^2 \alpha_t / \alpha_{t-1}}{\sall^2} \\ 
     &= 1 - \frac{(1 - \alpha_{t-1})\alpha_t / \alpha_{t-1}}{1 - \alpha_t} \\
     &= \dfrac{1 - \alpha_t / \alpha_{t-1}}{1 - \alpha_t}\eqsp,
\end{align}
so that one recovers the famous equations made popular in the seminal DDPM paper \cite[Equation 7]{ho2020denoising}:
\begin{align} \label{app:eq:ddpm:reverse}
        \meank_{t-1} &= \dfrac{\sqrt{\alpha_{t-1}}}{\sqrt{\alpha_t}}\left(Y_t - \dfrac{1 - \alpha_t/\alpha_{t-1}}{\sqrt{1 - \alpha_t}}\epsilon_t(Y_t, Y_0) \right) \\
        \Sigmak_{t-1} &= \left(1 - \alpha_{t-1}\right)\dfrac{1 - \alpha_t / \alpha_{t-1}}{1 - \alpha_t}\eqsp,
\end{align}
with $\epsilon_t(Y_t, Y_0) = {(Y_t - \sqrt{\alpha_t} Y_0)}/{\sqrt{1 - \alpha_t}}$.

\paragraph{Model for the reverse process.}\ 
We propose approximating the backward process associated to $\{Y_t\}_{t=0}^T$, \emph{given} $\{A_t\}_{t=1}^T$, adapting the DDPM approach.
This time, for the backward process, we characterized the conditional density of $Y_{t-1}$ given $Y_0, Y_{t}$ and $\listTA$: 
\begin{equation}
    \ikArg{1:T | 0, a}(y_{1:T} | y_0, a_{1:T}) := \pArg{T}(y_T)\prod\nolimits_{t=T}^{1} \kArg{t-1|0,t, a}(y_{t-1}|y_{t}, y_0, a_{1:t})\eqsp,
\end{equation}
where $\kArg{t-1|0,t, a}(\cdot|y_{t}, y_0, a_{1:t})$ is the tractable density of a Gaussian distribution, as we have just proved in \Cref{app:prop_backward}.
We similarly reconsider the family of Gaussian variational approximation introduced in \eqref{eq:gaussian_variational}, modified to account for an \iid\ sequence $\{A_t\}_{t=1}^T$:
\begin{equation}
\label{eq:stable_variational}
\qArg{t-1|t,a}(y_{t-1} | y_{t}, a_{1:t}) = \upphi_d(y_{t-1} | \meanq_{t-1}(y_{t}, a_{1:t}), \Sigmaq_{t-1}(y_{t}, a_{1:t})) \eqsp,
\end{equation}
where $\upphi_d$ is the density of the multivariate Gaussian distribution, so that the overall model for the backward process is the following:
\begin{equation}
\qArg{0:T}(y_{0:T}) =  \int \uppsia_{1:T}(a_{1:T}) p_T(y_T)\prod\nolimits_{t=1}^T \qArg{t-1|t,a}(y_{t-1} | y_{t}, a_{1:t}) \rmd a_{1:T}\eqsp,
\end{equation}
where $\uppsia_{1:T}(a_{1:T}) = \prod_{i=1}^T \uppsia(a_i)$, and $p_T$ is the density of $\stable_{\alpha}^{\text{i}}(0, \sall[T] \Idd)$.

\subsection{Training Loss}  \label{app:dlpm_training_loss}

In this section, we draw inspiration from DDPM (\cite{ho2020denoising}) to obtain a loss function admitting a closed-form formula. We further provide three design choices  which lead to a simplified training loss, corresponding to the denoising loss for $\alpha$-stable diffusion.

\subsubsection{Classical loss for DDPM, $\alpha = 2$}
We start by reviewing what is classically done for DDPM, \ie, the case $\alpha = 2$. The variational approximation $\{\qArg{0:T} \,: \, \theta \in\Theta\}$, for some parameter space $\Theta$,  is designed to admit the same decomposition as $\ikTwoArg{0:T}$, \ie, $\qArg{0:T}(x_{0:T}) = \qArg{T}(x_T) \prod_{t=T}^{1} \qArg{t-1|t}(x_{t-1}|x_{t})$, where $\qArg{T}$ is chosen to be the density of $\stable^{\text{i}}_{\alpha}(0, \sall \Idd)$ as an approximation of $\pArg{T}$. Then, it is trained on a classical upper bound of the KL loss $ \lossDDPM : \theta \mapsto \KL(\pstar \|  \qArg{0})$ between the true and the generated distribution, which is a form of evidence lower bound loss \cite[Equation 5]{ho2020denoising}. Thus one resorts to optimize the following sum:
\begin{equation} \label{app:eq:ddpm:loss}
    \lossDDPM(\theta) \leq \lossDDPM_{T} + \sum_{t=2}^{T}\lossDDPM_{t-1}(\theta) + \lossDDPM_{0}(\theta) + C
\end{equation}
where $C$ is a constant that does not depend on $\theta$, and
\begin{align}
    \lossDDPM_{T} &= \PE \left[ \KL\left(\kTwoArg{T|0}(\cdot | X_0) \ \| \ \upphi_d(\cdot | 0, \sall[T] \Idd)\right) \right] \\
    \lossDDPM_{0}(\theta) &= - \PE \left[ \log\left(\qArg{0|1}(X_0 | X_1)\right)\right] \\
    \lossDDPM_{t-1}(\theta) &= \PE \left[ \KL\left(\kTwoArg{t-1|0,t}(\cdot|X_0,X_{t}) \ \| \ \qArg{t-1|t}(\cdot|X_t)\right) \right]\eqsp,
\end{align}
where $\listT{X_t}$ is the process defined in \eqref{app:eq:dlpm:forward}, and $\kTwoArg{t-1|0,t}$ is the conditional density of $X_{t-1}$ given $X_0, X_t$.
We make the following classical remarks on the terms of this loss (\cite{sohldickstein2015deep}, \cite{yang2024diffusion}). The term $\lossDDPM_{T}$ does not depend on $\theta$ but only on the chosen time horizon for the forward process, that determines the final variance of the Gaussian distribution $\normal(0, \sall[T] \Idd)$. It is neglected. The effect of optimizing the first term $\lossDDPM_{0}(\theta)$ is negligible too.

More importantly, for the term $\lossDDPM_{t-1}(\theta)$, when using
Gaussian variational approximations, \ie, as
\begin{equation}
    \qArg{t-1|t}(x_{t-1} | x_{t}) = \upphi_d \left(x_{t-1} | \meanq_{t-1}(x_{t}), \Sigmaq_{t-1}(x_{t}) \right) \eqsp,
  \end{equation}
where $(x,\mtt,\Sigma) \mapsto  \upphi_d(x|\mtt,\Sigma)$ is the $d$-dimensional  density of the Gaussian distribution with mean $\mtt$ and covariance matrix $\Sigma$, $\meanq_{t-1},\Sigmaq_{t-1}$ are some functions of $x_{t}$ parameterized by $\theta$, it turns out that $\lossDDPM_{t-1}$ admits a closed-form expression. For a fixed variance $\Sigmaq_{t-1} = \Sigmak_{t-1}$, with $\Sigmak_{t-1}$ given in \eqref{app:eq:ddpm:reverse}, one resorts to optimize a convenient $\mathbb{L}_2$ loss function:
\begin{equation}
\label{eq:ddpm_training}
    \lossDDPM_{t-1}(\theta) = \lambda_t \| \meank_{t-1}(x_{t}, x_0) - \meanq_{t-1}(x_{t}) \|^2,
\end{equation}
where $\lambda_t, \meank_t$ depend on the noise schedule $(\gstep, \sstep)$ and $x_{t}, x_0$.

Unfortunately, as we mentioned in \Cref{data_augmentation_approach}, this solution cannot be used as such to learn the backward transitions associated to  $\{X_t\}_{t=0}^{T}$ for $\alpha<2$. The main issue that we face stems from the fact that the density of $\alpha$-stable distributions are in most cases unknown, in contrast to Gaussian distributions. 
As a result, the conditional density $x_{t-1},x_0,x_{t} \mapsto \kArg{t-1|0,t}(x_{t-1}|x_0,x_{t})$ is unknown for $\alpha<2$, which prevents us to have an explicit expression for $\theta \mapsto \lossDDPM_{t-1}(\theta)$.

Moreover, the absence of a second order moment for $\alpha$-stable distributions challenges the most straightforward adaptation we can make to the previous loss considering the data augmentation setting. Indeed, to fit $\theta$ to the data distribution, we aim to rely on Kullback-Leibler minimization, a.k.a. the maximum likelihood principle, and some associated upper bounds. A naive solution would consist in considering the bounds obtained applying the Jensen inequality:
\begin{equation}
\label{eq:6}
\txts \KL(\pstar \|  \qArg{0}) \leq   \mathbb{E} \left[\KL\left[\pstar(\cdot) \| \qArg{0|a}(\cdot |A_{1:T})\right] \right]\eqsp,
\end{equation}
and fall back to the expression obtained in \eqref{app:eq:ddpm:loss}, only with conditioning on $\listTA$. However, as we see in \eqref{eq:ddpm_training}, this expression would involve taking expectation of $A_t$, while it is distributed as $\stableA$ and admits no first order moment. We are not aware of any bounds on $\KL(\pstar \parallel \qArg{0|a}(\cdot |A_{1:T}))$ that would lead to a meaningful optimization problem due to the heavy tailed nature of the distribution of $\listTA$.

\subsubsection{Loss for DLPM, for any $\alpha$} \label{app:dlpm:loss:intuitive}
We now expose our methodology to address this limitation. We keep the structure of the classical loss and aim at minimizing the error between the backward process $\listT{\iY_t}$ and its variational approximation $\listT{\iYArg_t}$. To do so we consider the same loss structure as before, but take the square root of individual KL terms. See \Cref{app:principled_approach} for a more principled approach leading to a similar loss. Thus we consider the following valid loss function:
\begin{equation}\label{app:eq:loss_function}
    \lossDLPMD : \theta \mapsto \lossDLPMD(\theta) =  \mE \left[\sum_{t=2}^{T} \left(\lossDLPMD_{t-1}(\theta, A_{1:t})\right)^{r}\right]\eqsp,
\end{equation}
where $r>0$,
\begin{equation}
    \lossDLPMD_{t-1}(\theta, A_{1:t}) = \mE \left[ \KL\left( \kArg{t-1 | t,0,a} (\cdot | Y_t, Y_0, A_{1:t}) \ \| \ \qArg{t-1|t,a} (\cdot | Y_t, A_{1:t})\right) \ \bigg| \ A_{1:t} \right]\eqsp,
\end{equation}
and $\kArg{t-1|0,t, a}$ denotes the conditional density of $Y_{t-1}$ given $Y_0, Y_{t}$ and $\listTA$.
In order to ensure that the expectations with respect to $A_{1:T}$ are finite, we need to choose $r < \frac{\alpha}{2}$ when $\alpha \in (1,2)$. For simplicity, in the rest of the paper, we will use $r=\frac{1}{2}$.



\begin{proposition}[Training loss for DLPM] \label{app:prop_dlpm_initial_loss}
The loss $\lossDLPMD$ admits a closed-form expression, such that one resorts to optimize the following loss for $1 \leq t \leq T$:
\begin{equation}
\label{app:eq_dlpm_loss_full}
    \lossDLPMD_{t-1}(\theta, A_{1:t}) = \mE \bigg[ \frac12 \log\dfrac{\Sigmaq_{t-1}}{\Sigmak_{t-1}} + \dfrac{\Sigmak_{t-1} + \| \meank_{t-1} - \meanq_{t-1} \|^2 }{2\Sigmaq_{t-1}}  - \dfrac{1}{2} \ \bigg| \ A_{1:t} \bigg]
\end{equation}
where 
\begin{align}
    \meank_{t-1}(Y_t, Y_0, A_{1:t}) &= \frac{1}{\gstep[t]}\left(Y_{t} - \sall[t] \GammaA \epsilon_t(Y_t, Y_0)\right) \\ 
    \Sigmak_{t-1}(A_{1:t}) &= \GammaA \SallA[t-1] \\ 
    \epsilon_{t}(Y_t, Y_0) &= \frac{Y_{t} - \gall Y_0}{\sall} \\ 
    \SallA[t] &= \sum_{k = 1}^t \left(\frac{\gall[t]}{\gall[k]}\sqrt{A_k}\sstep[k]\right)^{2} \\ 
    \GammaA &= 1 - \frac{\gstep^2 \SallA[t-1]}{\SallA[t]}\eqsp,
\end{align}
where $\meanq_{t-1}, \Sigmaq_{t-1}$ are the mean and variance of the backward transition kernels $\qArg{t-1 | t}$. We have omitted the arguments of the mean and variance functions for clarity.
\end{proposition}
 
\begin{proof}
Recall (\Cref{app:prop_backward}) that the backward process $Y_{t-1}$ conditioned on $\listTA, Y_t, Y_0$ at time $t$ is distributed as $\normal(\meank_{t-1}, \Sigmak_{t-1})$, and, by design (\Cref{data_augmentation_approach}), the backward transition kernels $\qArg{t-1 | t, a}$ of each element of the variational family describe a Gaussian transition kernel of mean $\meanq_{t-1}$ and variance $\Sigmaq_{t-1}$ at time $t$, as defined in \ref{app:backward_formula}.
Since the KL term in $\lossDLPMD_{t-1}(\theta, A_{1:t})$ corresponds to a KL divergence between two Gaussian distributions, a closed-form formula is readily available. Here we rewrite the equation with all the functions arguments written out explicitly:
\begin{align}\label{app:eq:full_loss}
    \lossDLPMD_{t-1}(\theta, A_{1:t}) = \mE \bigg[&\frac12 \log\dfrac{\Sigmaq_{t-1}(A_{1:t})}{\Sigmak_{t-1}(A_{1:t})} - \dfrac{1}{2} \\
    & + \dfrac{\Sigmak_{t-1}(A_{1:t}) + \| \meank_{t-1}(Y_{t}, Y_0, A_{1:t}) - \meanq_{t-1}(Y_{t}, A_{1:t}) \|^2 }{2\Sigmaq_{t-1}(A_{1:t})} \ \bigg| \ A_{1:t} \bigg]
\end{align}
\end{proof}


Now we discuss further design choices for $\lossDLPMD_{t-1}(\theta, A_{1:t})$, $\qArg{t-1|t, a}$, leading to a simplified denoising loss, as is usually done in the literature. We denote them by \textbf{D1, D2} and \textbf{D3}.

\hypertarget{D1}{\paragraph{D1.}\ } We set a fixed variance $\Sigmaq_{t} = \Sigmak_t$ for the reverse process
, but we expect a study on the effect of learning variance to yield similar results as in original DDPM (\cite{ho2020denoising}), and especially its improved version (\cite{nichol2021improved}).

\hypertarget{D2}{\paragraph{D2.}\ } Following our own experimental results and the usual recommendation for denoising diffusion models (see, \eg, \cite{yang2024diffusion, karras2022elucidating}), we reparameterize the output of the model to predict the value $\epsilon_t(y_t, y_0)$ at time-step $t$ rather than $\meank_{t-1}(y_t, y_0, a_{1:t})$. 
Since 
\begin{equation}
\meank_{t-1}(Y_t, Y_0, A_{1:t}) = \frac{1}{\gstep[t]}\left(Y_{t} - \sall[t] \GammaA \epsilon_{t}(Y_t, Y_0)\right)\eqsp,
\end{equation}
we re-parameterize $\meanq_{t-1}$ to be equal to
\begin{equation} \label{app:eq:m_to_eps}
    \meanq_{t-1}(Y_t, A_{1:t}) = \frac{1}{\gstep[t]}\left(Y_{t} - \sall[t] \GammaA \epsilonq_{t}(Y_t, A_{1:t})\right)\eqsp.
\end{equation}
with $\epsilonq_t$ being the output of the model. 
Following experimental results (see the introduction of \Cref{app:experiment}), we drop the dependency of $\epsilonq_t$ on $\listTA$, which is reasonable since it approximates $\epsilon_t : (Y_t, Y_0) \mapsto \epsilon_t(Y_t, Y_0)$. Thus $\epsilonq_t$ only depends on $t, Y_t$. This choice achieves better performance in our experiments, allows further computational tricks (introduced in \Cref{app:less_rv}), and enables one to re-use existing neural network architectures.

\hypertarget{D3}{\paragraph{D3.}\ } 
Assuming \hyperlink{D1}{D1}, \hyperlink{D2}{D2}, $\lossDLPMD_{t-1}(\theta)$ becomes
\begin{equation} \label{app:eq:simplified_loss_t}
    \lossDLPMD_{t-1}(\theta)
    =  \mE \left[ \lambda_{t, A_{1:t}}^{2} \| \epsilonq_{t}(Y_t, A_{1:t}) -\epsilon_t(Y_t, Y_0) \|^{2} \right]\eqsp,
\end{equation}
where $\lambda_{t, a_{1:t}} = {\Gammaa \sall[t]}/{2\gstep[t]\Sigmak_{t-1}}$, and $\epsilon_{t}(Y_t, Y_0)= {(Y_{t} - \gall Y_0)} / {\sall}$.
The methodological knowledge of diffusion models motivates making specific choices for $\lambda$, different from its defined value, resulting in a classical technique for improving performances (\eg, see \cite{karras2022elucidating, ho2020denoising, nichol2021improved, yang2024diffusion}). We will stick to the classical choice of choosing $\lambda_{t, a_{1:t}}=1$, which experimentally works better and draws similarities to the score-based perspective  (see \Cref{app:lim}). Other choices and optimizations are left to further work.


The proof of \Cref{prop:denoising} follows immediately from these design choices, hence omitted.

\subsubsection{A principled approach for deriving the loss function} \label{app:principled_approach}
In this section, we provide a more principled approach to derive the loss function for DLPM, as initially given in \eqref{app:eq:loss_function}. We show the derivation for $r=1/2$; however, the same derivation applies for any $r\in (0,1]$.

Noting that $Y_0$ is independent of $\{A_t\}_{t=1}^T$ in \eqref{eq:forward_t_with_a},  $\pstar$ is the equal to $\kArg{0|a}(\cdot |a_{1:T})$, the conditional density of $Y_0$ given $A_t=a_t$ for any $t\in\{1,\ldots,T\}$, and therefore, we consider the valid loss function 
\begin{equation}
  \label{app:principled_loss}
 \lossDLPMD(\theta) : \theta \mapsto  \int \rmd a_{1:T} \uppsia_{1:T}(a_{1:T}) \left[ \KL(\kArg{0|a}(\cdot |a_{1:T}) \| \qArg{0|a}(\cdot |a_{1:T}))\right]^{1/2}\eqsp.
\end{equation}


While this function is still intractable, we can provide an upper bound which we can minimize. Indeed, using Jensen inequality twice, we bound this function by 
\begin{align}
  \label{eq:8}
  &\lossDLPMD(\theta)  = \txts \int  \rmd a_{1:T} \uppsia_{1:T}(a_{1:T}) \defEns{\int \rmd y_0  \kArg{0|a}(y_0 |a_{1:T})  (\log  \kArg{0|a}(y_0 |a_{1:T}) - \log  \qArg{0|a}(y_0 |a_{1:T}) ) }^{1/2} \\
  & \txts\leq \int  \rmd a_{1:T} \uppsia_{1:T}(a_{1:T}) \defEns{- \int  \rmd y_{0:T}  \kArg{0:T|0,a}(y_{0:T} |a_{1:T})  \log  \frac{  \qArg{0:T|a}(y_{0:T}|a_{1:T})}{    \kArg{1:T|0,a}(y_{1:T} |y_0,a_{1:T})}  +\Cst_1 }^{1/2} \\
& = \txts\int  \rmd a_{1:T} \uppsia_{1:T}(a_{1:T}) \defEnsLigne{\sum_{t=0}^{T-1} \lossDLPM_t(\theta, a_{1:T}) + \Cst_1+\Cst_2 }^{1/2} \\
& \leq \txts\int  \rmd a_{1:T} \uppsia_{1:T}(a_{1:T}) \defEnsLigne{\sum_{t=0}^{T-1} \lossDLPM_t(\theta, a_{1:T})^{1/2} + \Cst_1^{1/2}+\Cst_2^{1/2} } \eqsp ,
\end{align}
where we used $\sqrt{a + b} < \sqrt{a} + \sqrt{b}$ when $a, b \geq 0$ and $    \lossDLPM_{0}(\theta, a_{1:T}) = {\mE} [-\log p_{\theta}(Y_0 | Y_1, a_{1:T}) | \{A_t\}_{t=1}^T = \{a_t\}_{t=1}^T ]$  for $t >0$,
\begin{equation}
  \label{eq:def_loss_t}
    \lossDLPM_{t}(\theta, a_{1:T}) = {\mE} \left[ \KL( \kArg{t | t+1,0,a} (\cdot | Y_t, Y_0,  a_{1:T}) \| \qArg{t|t+1,a} (\cdot | Y_t,  a_{1:T})) | \{A_t\}_{t=1}^T = \{a_t\}_{t=1}^T\right] \eqsp,
\end{equation}
and $\Cst_1  = \int \rmd y_0 \pstar(y_0) \log \pstar(y_0) \rmd y_0$ and $\Cst_2 = {\mE} [ \KL(\kArg{1:T|0,a}(\cdot|Y_0, a_{1:T}) | \qArg{T} ) | \{A_t\}_{t=1}^T = \{a_t\}_{t=1}^T]$ does not depend on $\theta$ since $\qArg{T}$ is chosen as $\stable_{\alpha}(0, \sall \Idd)$. 
Regarding $\lossDLPM_{0}$, we neglect this term, replacing the distribution $\qArg{0|1,a}(\cdot|y_1,a_{1:T})$ by a deterministic Dirac. One could alternatively employ the strategy of the discrete decoder for image data as described by \cite{ho2020denoising}.  We end up with the final loss function:
\begin{equation}
  \lossDLPM(\theta) =  \txts\int  \rmd a_{1:T} \uppsia_{1:T}(a_{1:T}) \defEnsLigne{\sum_{t=0}^{T-1} \lossDLPM_t(\theta, a_{1:T})^{1/2} } \eqsp.
\end{equation}
We can then provide an explicit expression for $\lossDLPM_t(\theta, a_{1:T})$ based on the result of \Cref{app:prop_backward}. 



\subsection{Reducing the computational cost with faster sampling at each timestep}
\label{app:less_rv}

In this section, we provide a faster algorithm for training DLPM, computing each loss term $\lossDLPMD_{t-1}$ using only two heavy-tailed random variables per datapoint, instead of $t$ random variables.

Consider again the process $\listT{Y_t}$. We replace the loss function \eqref{app:eq:loss_function} by an equivalent one:
\begin{equation}\label{app:eq:loss_function_time}
    \lossDLPMD_{\text{time}}(\theta) =  \mE \left[ \lossDLPMD_{t-1}(\theta, A_{1:t})^{1/2}\right]\eqsp,
\end{equation}
where $t \sim \mathcal{U}[2, T]=(\sum_{i = 2}^T \delta_i)/(T-1)$. 
The standard technique for computing the loss consists in the following loop:
\begin{enumerate}
    \item Take a batch of $B$ datapoints $\{Y_0^i\}_{i=1}^B$.
    \item For each datapoint $Y_0^i$, draw a random $t_i$,as suggested by the alternative loss \eqref{app:eq:loss_function_time}. 
    Indeed, for a single datapoint, (i) training on all timesteps rather than just one yields equal to inferior results, for a much higher computational cost, (ii) it is beneficial for the model to proportionally spend more time learning specific time ranges (iii) thus the distribution of $t$ can be optimized and is a matter of ongoing methodological research, \eg, see \cite{karras2022elucidating}.

    \item Draw sequences $\{A_t^i\}_{t=1}^{t_i}$ of heavy-tailed random variables.
    \item Compute the noised datapoints $\{Y_{t_i}^i\}_{i=1}^{B}$.
    \item Compute the batch loss 
    \begin{equation}
        \hat \lossDLPMD(\theta) = \frac{1}{B}\sum_{i=1}^B \lossDLPMD_{t_i-1}(\theta, Y_{t_i}^i, A_{1:t_i}^i)\eqsp,
    \end{equation}
    such that $\hat \lossDLPMD(\theta) \approx \lossDLPMD_{\text{time}}(\theta)$.
    \item Do an optimization step.
\end{enumerate}

Step $3$ can be expensive, since one has to sample on average $O(T)$ $d$-dimensional heavy-tailed random variables to compute a single noised datapoint $Y_t$ from $Y_0$. This is all the more inefficient as $T$ can be quite large (indeed, on image datasets we can have $T = 4000$, see \Cref{app:experiment}). 

One can guess that this is abusive, especially since characterizing the distribution of $Y_t$ given $Y_0$ only requires a single heavy-tailed random variable:
\begin{equation}
    Y_t \eqd \gall Y_0 + \sall \bar{A}^{1/2} \bar G_t\eqsp,
\end{equation}
where $\bar A\sim \stableA$, $\bar G_t \sim \normal(0, \Idd)$.
As we formalize in the next proposition, it is indeed possible to bypass the sampling of a whole sequence. Since we manipulate the joint distribution of $(Y_0, Y_{t-1}, Y_t)$ for the loss term $\lossDLPMD_{t-1}(\theta)$, we will actually need to sample two heavy-tailed random variables.

\begin{proposition}[Sampling two heavy-tailed r.v for each loss term]
\label{app:prop_training_sampling_less} 
Suppose that the functions
$\meanq_{t-1}, \Sigmaq_{t-1}$ satisfy for any $y_t, a_{1:t}$:
\begin{align}
 \meanq_{t-1}(y_t, a_{1:t}) &= M^{\theta}_{t-1}\left(y_t, a_t, \frac{\Salla[t-1]}{\sall[t-1]^2}\right) \\ 
 \Sigmaq_{t-1} &= S^{\theta}_{t-1}\left(y_t, a_t, \frac{\Salla[t-1]}{\sall[t-1]^2}\right)\eqsp,
\end{align}
for some functions $M^{\theta}_{t-1}, S^{\theta}_{t-1}$, and where 
\begin{equation}
\Salla[t] = \sum_{k = 1}^t \left(\frac{\gall[t]}{\gall[k]}\sqrt{a_k}\sstep[k]\right)^{2}   \eqsp,
\end{equation}
as given in \eqref{app:eq:gall_Sall}.
Then each term $\mE [\lossDLPMD_{t-1}(\theta, A_{1:t})]^{1/2}$ of the loss can be computed sampling only two independent random variables $\barA[0], \barA[1]$ distributed as $\stableA$:
\begin{equation}
\mE \left[\lossDLPMD_{t-1}(\theta, A_{1:t})\right]^{1/2} =  \mE \left[\lossDLPMLess_{t-1}(\theta, \barA)\right]^{1/2}\eqsp,
\end{equation}
where $\barA := (\barA[0], \barA[1])$, and
\begin{align}\label{app:eq:full_loss_less}
    \lossDLPMLess_{t-1}(\theta, \barA) = 
    \mE \bigg[&\frac12 \log\dfrac{S^{\theta}_{t-1}(Z_{t}, \barA)}{\Sigmak_{t-1}(\barA)} - \dfrac{1}{2} \\
    & + \dfrac{\Sigmak_{t-1}(\barA) + \| \meank_{t-1}'(Z_t, Z_0, \barA) - M^{\theta}_{t-1}(Z_{t}, \barA) \|^2 }{2 S^{\theta}_{t-1}(Z_{t}, \barA)} \ \bigg| \ \barA \bigg]\eqsp,
\end{align}
with $\listT{Z_t}$ being a stochastic process defined as
\begin{equation}
    Z_0 = Y_0\eqsp, \quad Z_t = \gall Z_0 + \Sigma_t'^{1/2} \bar G_t\eqsp,
\end{equation}
where $\listTG$ is an \iid\ sequence distributed as $\normal(0, \Idd)$, and
\begin{align}
        \SallPtm &= \sall[t-1]^2 \barA[0] \\ 
        \SallPt &=  \sigma_t^2 \barA[1] + \gstep[t]^2 \SallPtm \\ 
        \GammaP &= 1 - \dfrac{\gstep[t]^2 \SallPtm}{\SallPt}\eqsp,
\end{align}
such that $Z_t \eqd Y_t$, and:
\begin{align}
\meank_{t-1}(Z_t, Z_0, \barA) &= \frac{1}{\gstep[t]}\left(Z_{t} - \sall[t]  \GammaP \epsilon_t(Z_t, Z_0)\right) \\ 
\Sigmak_{t-1}(\barA) &=\GammaP \SallPtm \\ 
\epsilon_{t}(Z_t, Z_0) &= \frac{Z_t - \gall Z_0}{\sall} \eqsp.
\end{align}

In order to keep the notations similar for all $t \geq 2$, in the case of $\lossDLPMLess_{1}$, we always set $\bar A^2_0 = 0$.
\end{proposition}
\begin{proof}
Remember the full equation for the loss, first given in \Cref{app:prop_dlpm_initial_loss} \eqref{app:eq:full_loss}:
\begin{align}
    \lossDLPMD_{t-1}(\theta, A_{1:t}) = \mE \bigg[&\frac12 \log\dfrac{\Sigmaq_{t-1}(A_{1:t})}{\Sigmak_{t-1}(A_{1:t})} - \dfrac{1}{2} \\
    & + \dfrac{\Sigmak_{t-1}(A_{1:T}) + \| \meank_{t-1}(Y_{t}, Y_0, A_{1:T}) - \meanq_{t-1}(Y_{t}, A_{1:T}) \|^2 }{2\Sigmaq_{t-1}(A_{1:t})} \ \bigg| \ A_{1:t} \bigg]\eqsp.
\end{align}
Now, all the required variables and computations only depend on $A_t, \Sall[t-1]$; this is the case for $\meanq_{t-1}, \Sigmaq_{t-1}$ by hypothesis, and this is the case for $\meank_{t-1}, \Sigmak_{t-1}$ as one can see in \eqref{app:eq:meank_Sigmak}. 
Rewriting the previous loss as
\begin{align}
    \lossDLPM_{t-1}(\theta, A_{1:t}) = 
    \mE \bigg[&\frac12 \log\dfrac{S^{\theta}_{t-1}\left(Z_{t}, A_{t}, \frac{\SallA[t-1]}{\sall[t-1]^2}\right)}{\Sigmak_{t-1}(A_{1:T})} - \dfrac{1}{2} \\
    & + \dfrac{\Sigmak_{t-1}(A_{1:t})}{2 S^{\theta}_{t-1}\left(Z_{t}, A_{t}, \frac{\SallA[t-1]}{\sall[t-1]^2}\right)} \\ 
    & + \dfrac{\| \meank_{t-1}(Y_t, Y_0, A_{1:t}) - M^{\theta}_{t-1}\left(Z_{t},  A_{t}, \frac{\SallA[t-1]}{\sall[t-1]^2}\right) \|^2 }{2 S^{\theta}_{t-1}\left(Z_{t}, A_{t}, \frac{\SallA[t-1]}{\sall[t-1]^2}\right)}
    \ \bigg| \ A_t, \Sall[t-1] \bigg]\eqsp,
\end{align}
it becomes clear how the expectation can be taken on the joint distribution of 
\begin{equation}
    \left(Y_0, Y_{t-1}, Y_t, A_t, \frac{\SallA[t-1]}{\sall[t-1]^2}\right)\eqsp.
\end{equation}
A direct application of \Cref{app:lemma:reduction} shows that this expectation can be taken on the joint distribution of the five random variables $(Z_0, Z_{t-1}, Z_t, \barA[1], \barA[0])$, which only necessitates sampling two heavy-tailed random variables $\barA[0], \barA[1]$. Using the formulas for $Z_0, Z_{t-1}$ and $Z_t$ given $\barA[0], \barA[1]$ as defined in \Cref{app:lemma:reduction}, we obtain the equivalent loss \eqref{app:eq:full_loss_less}.


\end{proof}

As we will prove in the next proposition, the conditions of \Cref{app:prop_training_sampling_less} are always satisfied under design choices \hyperlink{D1}{D1}, \hyperlink{D2}{D2}. Under design choice \hyperlink{D3}{D3}, we can also rewrite the simplified denoising loss given in \Cref{prop:denoising}.
\begin{proposition}[Sampling one heavy-tailed r.v in the simplified loss]
\label{app:prop_training_loss_simplified_less}
Assume the design choices \hyperlink{D1}{D1}, \hyperlink{D2}{D2}, \hyperlink{D3}{D3} are satisfied. Then one can obtain the following simplified denoising objective function from the full objective function given in \eqref{app:eq:full_loss_less}:
\begin{equation}
\label{app:eq:simplified_loss_less}
 \lossDLPMSimpleLess(\theta) = \sum_{t=1}^{T} \mE \left[\mE \left( \| \epsilonq_t(Z_{t}) - \bar A_t^{1/2} \bar G_t \|^{2} \ | \bar A_t \right)^{1/2}\right],
\end{equation}
where $\{\bar A_t\}_{t=1}^T$ is an \iid\ sequence distributed as $\stableA$, and
\begin{equation}
    Z_t = \gall Z_0 + \sall[t] \bar A_t^{1/2} \bar G_t\eqsp,
\end{equation}
with $\{\bar G_t\}_{t=1}^T$ an \iid\ sequence distributed as $\normal(0, \Idd)$.
\begin{proof}
    Let us recall design choice \hyperlink{D1}{D1}:
\begin{align}
    \Sigmaq_{t-1}(A_{1:}) &= \GammaA \SallA \eqsp,
\end{align}
and design choice \hyperlink{D2}{D2}:
\begin{align}
    \meanq_{t-1}(Z_t, A_{1:t}) &= \frac{1}{\gstep[t]}\left(Z_{t} - \sall[t] \GammaA \epsilonq_{t}(Z_t)\right)\eqsp.
\end{align}
Since $\Gamma$ only depends on $\Sall$ and $\Sall[t-1]$, both $\meanq_{t-1}, \Sigmak_{t-1}$ can be expressed as functions of $z_t, a_t$ and $\Salla$. Thus the assumptions of \Cref{app:prop_training_sampling_less} are satisfied. Using the same notations, we can apply the same algebraic transformations as in 
\eqref{app:eq:simplified_loss_t}, and by design choice \hyperlink{D3}{D3}, obtain:
\begin{equation}
 \lossDLPMLess(\theta) = \sum_{t=1}^{T} \mE \left[\mE \left( \| \epsilonq_t(Z_{t}) - \epsilon_t(Z_t, Z_0) \|^{2} \ | \barA[0], \barA[1] \right)^{1/2}\right]\eqsp.
\end{equation}
Finally, we apply \Cref{app:distrib_Sall} to $\Sigma_t'$ to affirm that $\Sigma_t' \eqd \sall[t]^{2}\bar A_{t}$, where $\bar A_t \sim \stableA$, and obtain the final loss we presented.
\end{proof}

\end{proposition}

We stress that this denoising training loss is similar to that of LIM \cite[Theorem 4.3]{NEURIPS2023_score_based_levy_lim}, but elevated to the necessary power to guarantee that the loss is finite. See \Cref{app:comparing_lim_dlpm} for a more detailed discussion.

\section{Denoising L\'evy Implicit Models (DLIM)} \label{app:dlim}
Using the same techniques as in DDIM (\cite{DDIM}), we obtain a deterministic sampling process which we naturally call Denoising Levy Implicit Models (DLIM).
Alike the Gaussian case treated in the original DDIM work, we will show that both DLPM and DLIM can share the same neural network.

\subsection{Non-Markovian forward process}
Let $\{\rho_t\}_{t=1}^T$ be an alternative noise schedule, proper to DLIM, that will ultimately tend to zero for deterministic generation. 
In the same way as in \Cref{data_augmentation_approach}, we take a data augmentation approach. 
We consider a process $\{Z_t\}_{t=1}^T$ defined by $Z_0 \sim \pstar$, where $\pstar$ is the data distribution, $Z_{T} \sim \stable(\gall[T] Z_0, \sall[T] \Idd)$ and, for $1 < t \leq T$
\begin{equation}
    Z_{t-1} = \gall[t-1] Z_0 + (\sall[t-1]^{\alpha}-\rho_t^{\alpha})^{1/\alpha} \epsilon_t(Z_t, Z_0) + \rho_t A_{t}^{1/2} G_t
\end{equation}
where $\epsilon_t(Z_t, Z_0) = (Z_t - \gall[t]Z_0) / \sall$, and $\{A_{t}\}_{t=1}^T$,  $\{G_t\}_{t=1}^T$ are independent random variables distributed according to $A_t \sim \stable_{\alpha/2,1}(0, c_A )$ and $G_t \sim \gauss(0,\Idd)$.

\begin{proposition} \label{app:prop_dlim_closed_transition}
The distribution of $Z_t$ given $Z_0$ is the same as that of $Y_t$ given $Y_0$.
\end{proposition}
\begin{proof}
This is a simple proof by induction, where one can re-adapt the technique of \cite[Lemma B.1]{DDIM} with the property for addition of $\alpha$-stable variable as we introduced in \Cref{app:prop:stab_alpha}. The case $t = T$ is true by construction. Suppose now that the property is verified at timestep $t$, where $1\leq t \leq T$. Then, focusing on the distribution of $Z_{t-1}$ given $Z_0$, $\epsilon_t(Z_t, Z_0)= (Z_t - \gall[t]Z_0) / {\sall}$ is distributed as $\stablesym$ by hypothesis, and thus by \Cref{app:prop:stab_alpha} and since $A_t^{1/2}G_t \sim \stablesym$, we can write
\begin{equation}
    Z_{t-1} \eqd \gall[t-1] Z_0 + \sall[t-1] \bar \epsilon_t \eqsp,
\end{equation}
where $\bar \epsilon_t \sim \stablesym$, which shows that indeed $Z_{t-1}$ given $Z_0$ admits the same distribution as $Y_{t-1}$ given $Y_0$.
\end{proof}

The design of this process makes the distribution of $Z_t$ given $Z_0$ match that of $Y_t$ given $Y_0$, where $\listT{Y_t}$ is the forward process of DLPM \eqref{eq:forward_t_with_a}. The task of sampling from it is thus efficient and straightforward.

\subsection{Generative process}
We similarly reconsider the family of Gaussian variational approximation introduced in \eqref{eq:gaussian_variational}, which accounts for an \iid\ sequence $\{A_t\}_{t=1}^T$:
\begin{equation}
\qArg{0:T}(y_{0:T}) =  \int \uppsia_{1:T}(a_{1:T}) p_T(y_T)\prod\nolimits_{t=1}^T \qArg{t-1|t,a}(y_{t-1} | y_{t}, a_{1:t}) \rmd a_{1:T}\eqsp,
\end{equation}
where $\uppsia_{1:T}(a_{1:T}) = \prod_{i=1}^T \uppsia(a_i)$, $\uppsia$ is the density of $\stableA$, and $p_T$ is the density of $\stable(0, \sall[T] \Idd)$. We set
\begin{equation}
\label{app:dlim:sampling}
    \qArg{t-1|t, a}(z_{t-1} | z_{t}, a_{1:t}) = \upphi_d \left(z_{t-1} | \meanq_{t-1}(z_{t}), \ \rho_t^2 a_t \right)\eqsp,
\end{equation}
with $\upphi_d$ being the density of the multivariate Gaussian distribution: the variance is fixed, determined by the alternative noise schedule $\{\rho_t\}_{t=1}^T$. For deterministic sampling, one will ultimately choose $\rho_t = 0$ for all $t$ and sample the chain $\listT{\iZArg_t}$ as follows:
\begin{equation}
    \iZArg_T \sim \stable(0, \sall[T] \Idd)\eqsp, \quad \iZArg_{t-1} = \meanq_t(\iZArg_t) \ \text{for} \ t \in \{T, \cdots, 1\}\eqsp.
\end{equation}

As $z_t$ is available as input, the model can be fit to approximate the value of $\epsilon_t(z_t, z_0)$, and we reparameterize $\meanq_t$ as follows:
\begin{equation}
\label{app:dlim:epsilonq}
    \meanq_t(z_t) = \frac{z_t - \sall \epsilonq_t(z_t)}{\gstep} + (\sall[t-1]^{\alpha}-\rho_t^{\alpha})^{1/\alpha} \epsilonq_t(z_t)\eqsp.
\end{equation}
This is alike design choice \hyperlink{D2}{D2} in \Cref{app:dlpm:loss:intuitive}.

\subsection{Loss function and equivalence with DLPM}
\label{app:dlim:equiv_dlpm}
We denote by $\hArg{t-1 | t, 0, a}$ the density of $Z_{t-1}$ given $Z_t$, $Z_0$ and $A_{1:T}$, which is the 
density of the Gaussian distribution with mean $\gall[t-1] Z_0 + (\sall[t-1]^{\alpha}-\rho_t^{\alpha})^{1/\alpha} \epsilon_t(Z_t, Z_0)$ and covariance $\rho_t^2 A_{t} \Idd$.
Since this distribution is now a given, we are inclined to use the loss function introduced in \eqref{eq:loss_function}, which is:
\begin{align}
     \lossDLPMD(\theta) &:=  \mE \left[\sum_{t=2}^{T} \left(\lossDLPMD_{t-1}(\theta, A_{1:t})\right)^{1/2}\right]\eqsp, \qquad \text{where} \\
    \lossDLPMD_{t-1}(\theta, A_{1:t}) &:= \mE \left[ \KL\left( \hArg{t-1 | t, 0, a_{1:t}} (\cdot | Z_t, Z_0,  A_{1:t}) \ \| \ \qArg{t-1|t,a} (\cdot | Z_t,  A_{1:t})\right) \Bigl|  A_{1:T} \right] \eqsp.
\end{align}
Since for $2 \leq t \leq T$, $\hArg{t-1 | t, 0, a}$ and $\qArg{t-1|t,a}$ are the densities of Gaussian distributions, we can analytically compute each term of the loss, as in \eqref{app:eq:loss_function}:
\begin{equation}
    \lossDLPMD_{t-1}(\theta, A_{1:t}) = 
\dfrac{1}{2 \rho_t^2 A_t}
\|\gall[t-1] Z_0 + 
(\sall[t-1]^{\alpha}-\rho_t^{\alpha})^{1/\alpha} \epsilon_t(Z_t, Z_0)
\ - \meanq_t(Z_t, A_{1:t}) \|^2
\end{equation}
where $\epsilon_t(Z_t, Z_0) = (Z_t - \gall[t]Z_0) / {\sall}$. Since the variance of the elements of our variational family $\{\qArg{0:T}\}$ have been designed to match that of the backward process given $Z_0, A_{1:T}$, the expression for the loss is readily in a simpler format. 
Finally, using the reparameterization given in \eqref{app:dlim:epsilonq}, The loss term $\lossDLPMD_{t-1}(\theta, A_{1:T})$ becomes:
\begin{equation}
    \lossDLPMD_{t-1}(\theta, A_{1:t}) = \lambda'_{t, A_{t} } \| \epsilon_t(Z_t, Z_0) - \epsilonq_t(Z_t) \|^2\eqsp,
\end{equation}
where $ \lambda'_{t, a_{t}} =  (\sall - (\sall[t-1]^{\alpha}-\rho_t^{\alpha})^{1/\alpha})^2 / (2 \rho_t^2 a_t) $. 
By comparing with the simpler DLPM loss \eqref{app:eq:simplified_loss_t} with design choices \hyperlink{D1}{D1}, \hyperlink{D2}{D2}, as introduced in \Cref{app:dlpm:loss:intuitive}, we realize we obtained the same loss term, with a different multiplicative factor $\lambda'_{t, a_{t}}$ instead of $\lambda_{t, a_{1:t}}$ in $\eqref{app:eq:simplified_loss_t}$. Finally, considering the alternative loss where $\lambda'_{t, a_{t}} = 1$ for all $t$, alike the design choice \hyperlink{D3}{D3} in \Cref{app:dlpm:loss:intuitive}, we fall back to the same simplified objective function obtained for DLPM:
\begin{equation}
 \lossDLPMSimple(\theta) = \sum_{t=1}^{T}\mE \left[ \mE \left( \| \epsilonq_t(Y_{t}) - \epsilon_{t}(Y_t, Y_0) \|^{2} \ \big| \ A_{1:t} \right)^{1/2}\right]\eqsp,
\end{equation}

\subsection{Cauchy DLIM}
\label{app:cauchy_dlim}
In the special case of a non-isotropic Cauchy distribution $(\alpha = 1)$, it is possible to bypass the data augmentation machinery, since there exists a closed-form formula for the $\KL$ divergence between two Cauchy distributions. 
Denote by $\Cauchy(\mu, \sigma)$ the one-dimensional Cauchy distribution centered at $\mu$ and of scale $\sigma$. Then \cite[Theorem 1]{chyzak2019closedformformulakullbackleiblerdivergence}:
\begin{equation}
\label{eq:kl_cauchy}
    \KL(\Cauchy(\mu_1,\sigma_1) \parallel \Cauchy(\mu_2, \sigma_2)) = \log \left(\frac{(\mu_1 - \mu_2)^2 + (\sigma_1 + \sigma_2)^2}{4 \sigma_1 \sigma_2}\right)\eqsp.
\end{equation}

The forward process $\{Z_t\}_{t=1}^T$ is defined with $Z_0 \sim \pstar$, where $\pstar$ is the data distribution, $Z_{T} \sim \stable(\gall[T] Z_0, \sall[T] \Idd)$ and, for $1 < t \leq T$
\begin{equation}
    Z_{t-1} = \gall[t-1] Z_0 + (\sall[t-1]^{\alpha}-\rho_t^{\alpha})^{1/\alpha} \epsilon_t(Z_t, Z_0) + \rho_t \epsa_t\eqsp,
\end{equation}
where $\epsilon_t(Z_t, Z_0) = (Z_t - \gall[t]Z_0) / \sall$, and $\{\epsa_t\}_{t=1}^T \sim \Cauchy(0, \Idd)^{\otimes T}$, where $\Cauchy(0, \Idd)$ is a $d$-dimensional Cauchy distribution with independent components, centered at $0$, of unit scale.
The distribution of $Z_t$ given $Z_0$ admits the same closed form expression given in \Cref{app:prop_dlim_closed_transition}.  We denote by $\hArg{t-1 | t, 0}$ the density of $Z_{t-1}$ given $Z_t$, $Z_0$.

Our generative process will be an element of a parameterized family of distributions admitting Cauchy transitions:
\begin{equation}
\qArg{0:T}(y_{0:T}) =  p_T(y_T) \prod\nolimits_{t=1}^T \qArg{t-1|t,a}(y_{t-1} | y_{t}) \eqsp,
\end{equation}
where $p_T$ is the density of $\stable(0, \sall[T] \Idd)$, and
\begin{equation}
\label{app:cauchy_dlim:sampling}
    \qArg{t-1|t}(z_{t-1} | z_{t}) = \mathrm{C} \left(z_{t-1} | \meanq_{t-1}(z_{t}), \ \rho_t \Idd \right)\eqsp,
\end{equation}
where $\mathrm{C}( \cdot | \meanq_{t-1}(z_{t}), \ \rho_t \Idd )$ is the density of the multivariate non-isotropic Cauchy distribution $\Cauchy(\meanq_{t-1}(z_t), \rho_t \Idd)$.

Instead of using the loss function $\lossDLPMD$ defined in \eqref{eq:loss_function}, we derive the loss via the conventional evidence lower bound (ELBO) approach (see, e.g., \cite{ho2020denoising}). Omitting extremal terms, this yields:
\begin{align}
     \LossCauchy(\theta) &:=  \sum_{t=2}^{T} \LossCauchy_{t-1}(\theta)\eqsp, \qquad \text{where} \\
    \LossCauchy_{t-1}(\theta) &:= \mE \left[ \KL\left( \hArg{t-1 | t, 0} (\cdot | Z_t, Z_0) \ \| \ \qArg{t-1|t} (\cdot | Z_t)\right)\right] \eqsp.
\end{align}
Using \eqref{eq:kl_cauchy}, we obtain a closed form formula for the final loss:
\begin{equation}
    \LossCauchy_{t-1}(\theta) = \sum_{i=1}^d  \log \left(\frac{\left(\meank_{i, t-1}(Z_t, Z_0) - \meanq_{i, t-1}(Z_t)\right)^2}{4 \rho_t^2} + 1\right)\eqsp,
\end{equation}

which could also serve as a template for another family of losses for heavy-tailed diffusion models. 
We leave these methodological explorations and possible extensions to the isotropic Cauchy case for future work. Based on our experimental findings, we expect an isotropic implementation to significantly outperform a non-isotropic one.

We outline again that such simplifications are not available for DLPM, since we are not able to characterize the distribution of $X_{t-1}$ given $X_{t}, X_0$ from the forward process $\{X_t\}_{t=0}^T$.

\section{Additional Information on Levy-Ito Models (LIM)}
\label{app:lim}

Here we briefly recapitulate the work done by \cite{NEURIPS2023_score_based_levy_lim}, introducing continuous diffusion models with $\alpha$-stable heavy-tailed noise. Using notations closer to ours, we define the noising schedule as any locally bounded continuous functions $\gamma : (t, X) \mapsto \gamma(t, X)$ and $\sigma : (t) \mapsto \sigma(t)$. We denote by $\levy_t$ the Levy process for which the increments between time $s<t$ follow a symmetric isotropic $\alpha$-stable distribution $
\stable^{\text{i}}_{\alpha}(0, (t-s)\Idd)$.
In this setting, the forward process $X_t$, with $X_0 \sim \pstar$, is written
\begin{equation}\label{app:eq:lim:forward}
    \rmd X_t = \gamma(t, X_{t-}) \rmd t + \sigma(t) \rmd \levy_t\eqsp,
\end{equation}
where $X_{t-}$ denotes the left limit of $X$ at time $t$. Similarly, $X_t$ is distributed as $\stable^{\text{i}}_{\alpha}(\gall X_0, \sall \Idd)$ when using Euler steps. This defines the cadlag (right continuous with left limits) solution, which in the case of $\alpha < 2$ a.s admits discontinuous jumps. We then consider the following backward process $\iX_t$:
\begin{equation} \label{app:eq:lim:backward}
    \rmd \iX_t = \left( - \gamma(t, \iX_{t+}) - \alpha \sigma^{\alpha}(t, \iX_{t+})\score_t(\iX_{t+}) \right)dt + \sigma(t) \rmd \bar \levy_t + d\bar{Z}_t
\end{equation}
where $\bar{Z}_t$ is the backward version of a Levy-type stochastic integral $Z_t$ s.t $\mE[Z_t] = 0$ with finite variation, and $\score_t$ is the fractional score function, defined to be
\begin{equation}
    \score_t(x) = \dfrac{\Delta^{\frac{\alpha - 2}{2}}\nabla p_t(x)}{p_t(x)}\eqsp,
\end{equation}
where $\Delta^{\eta / 2}$ denotes the fractional Laplacian of order $\eta / 2$ (\cite{Ortigueira_2014}). More precisely, $\Delta^{\eta / 2}f(x) = \mathcal{F}^{-1}\{\| u\|^{\eta} \mathcal{F}\{f\}(u) \}$, where $\mathcal{F}, \mathcal{F}^{-1}$ are the Fourier and inverse Fourier transforms.

The training loss is obtained using the classical technique of denoising score matching (\cite{vincent_denoising_score_matching}), where the following losses
\begin{equation}\label{app:eq:lim:denoising_loss}
    L : \theta \mapsto \mE \| s_{\theta}(X_t, t) - \score_t(X_t) \|^2\eqsp, \qquad L' : \theta \mapsto \mE \| s_{\theta}(X_t, t) - \score_t(X_t |X_0) \|^2\eqsp,
\end{equation}
are proven to be equivalent objective functions, with $s_{\theta}$ being the score approximation given by the model.



\subsection{Comparing LIM and DLPM} \label{app:comparing_lim_dlpm}
Let $(X_t)_{0\leq t \leq T}$ be the forward process introduced in \eqref{app:eq:lim:forward}. As stressed initially, the framework of LIM is not straightforward to manipulate, thus we do not characterize explicitly the distribution of $X_t$ given $X_0$ for an arbitrary noise schedule in the continuous case. Since the work done for LIM by \cite{NEURIPS2023_score_based_levy_lim} only provides the formulas for the scale-preserving schedule, we stick to them in the following: we keep the notation $\gall[t], \sall[t]$ for the continuous time regime equivalent of the scale preserving schedule we introduce in \Cref{app:experiment}, and they match on integer times $t$.

Considering an Euler scheme to obtain discretization for the forward and backward process, and using our own notations, both LIM and DLPM admit the same forward process $\{X_t\}_{t=1}^T$, $X_0 \sim \pstar$ and 
\begin{equation}
    X_t = \gstep[t] X_{t-1} + \sstep[t] \epsa_t,
\end{equation}
where $\{\epsa_t\}_{t=1}^T$ is an iid sequence of random variable distributed as $\stablesym$. We denote by $\{\iXArg_t \}_{t=T}^{0}$ the backward process associated to the Euler discretization of \eqref{app:eq:lim:backward}, where we use a neural network $s_{\theta}$ to approximate the true score $\score_t$. Since the true score of the data $\score_t(x_t | x_0)$ can be expressed as 
\begin{equation}
    \score_t(x_t | x_0) = -\dfrac{1}{\alpha \sall[t]^{\alpha - 1}} \epsilon_t(x_t, x_0)\eqsp,
\end{equation}

where $\epsilon_t(x_t, x_0) = (x_t - \gall[t] x_0) / {\sall[t]}$, we write
\begin{equation}
    s_{\theta}(x_t, t) = -\dfrac{1}{\alpha \sall[t]^{\alpha - 1}} \epsilonq_t(x_t, x_0)\eqsp,
\end{equation}
so that we rather work with $\epsilonq_t$, with the same intention that led us to the design choices given in \Cref{app:dlpm:loss:intuitive}.

Moreover, we denote by $\listT{\iYArg_t}$ the backward process of DLPM, as introduced in \eqref{eq:stable_variational}. As emphasized in \Cref{tab:samplers}, the sampling strategies for LIM and DLPM differ fundamentally when $\alpha \neq 2$. This is also the case for the training procedure.

\paragraph{Stochastic sampling.}\ 
The DLPM approach introduces the bounded random variable $0\leq \Gamma_t \leq 1$, interacting with the mean and variance of the Gaussian conditional at hand. Three points: when $\alpha=2$, $\Gamma_t$ becomes deterministic and one recovers DDPM formulas. Second, $\Gamma_t$ brings additional stochasticity in the sampling process. Third, it does so in the interesting manner than it simultaneously scales both (i) the magnitude of the noise added at time $t-1$ and (ii) the output of the noise model. 

\paragraph{Deterministic sampling.}\
In the case of the scale-preserving schedule, these two equations do not describe the same sampling procedure.

\begin{table}[ht!]
\centering
\renewcommand{\arraystretch}{2.0} 
    \begin{tabular}{lcc}
         & Stochastic & Deterministic \\ \hline  
    Continuous (LIM)
        &  
        $ 
        \dfrac{{\textcolor{blue}{\iXArg_t}}}{\gstep} - \dfrac{\alpha(1/\gstep - 1)}{\sall^{\alpha - 1}} {\textcolor{blue}{\epsilonq_t}}
        + (\dfrac{1}{\gstep^{\alpha}} - 1)^{1/\alpha} \epsilon_t' $
    & $\dfrac{\textcolor{blue}{\iXArg_t}}{\gstep} -
\left(\dfrac{\sall^{1 - \alpha}}{\gstep} - \sall^{1 - \alpha}\right){\textcolor{blue}{\epsilonq_t}}$

        \\ \hline
        
        Denoising (DLPM) 
        &  $\dfrac{\textcolor{blue}{\iYArg_t}}{\gstep} - {\textcolor{red}{\Gamma_t}}\sall {\textcolor{blue}{\epsilonq_t}} + 
         \sqrt{{\textcolor{red}{\Gamma_t} \Sall[t-1]}} G_t'
        $
        & $\dfrac{\textcolor{blue}{\iYArg_t}}{\gstep} - \left(\dfrac{\sall}{\gstep} - \sall[t-1]\right){\textcolor{blue}{\epsilonq_t}} $
        \\ \hline 
    \end{tabular}
    
    \caption{Distribution of $\iXArg_{t-1}, \iYArg_{t-1}$. $\{G_t'\}_{t = T}^1$ are independent random variables distributed as $\normal(0, \Idd)$, $\{\epsilon_t'\}_{t=T}^1$ are independent random variables distributed as $\stablesym$. 
    ${\textcolor{blue}{\epsilonq_t}}$ is the model at hand at time $t$, the formula for $\Sall[t]$ is given in \eqref{app:eq:gall_Sall}, and ${\textcolor{red}{\Gamma_t}}  = 1 - \gstep^2{\Sall[t-1]} / {\Sall[t]}$. Eventhough ${\textcolor{red}{\Gamma_t}}$ involves two heavy-tailed random variables, it is bounded: $0\leq {\textcolor{red}{\Gamma_t}}\leq 1$ (see \Cref{app:backward_process}).}
    \label{tab:samplers}
\renewcommand{\arraystretch}{1.} 
\end{table}

\paragraph{Training.}\ Alike the Gaussian case ($\alpha=2$), the score $\score_t(x_t | x_0)$ is a linear expression of the noise term $\epsilon_t(x_t, x_0)$, so the training equations are very similar, and can be reformulated to involve a denoising loss:
\begin{equation}
\label{app:eq:lim_vs_dlpm}
    \mathcal{L}_{t-1} : \theta \mapsto \mE \left(\| \epsilonq_t(X_t) - \epsilon_t(X_t, X_0) \|_p^{\eta} \right).
\end{equation}

\begin{itemize}
    \item In the case of DLPM, our discussion leads us to the choice $p = 2$ and $\eta = 1$ (see \eqref{eq:simple_loss}).
    \item In the case of LIM, the theory must rely on the choice $p = 2$ and $\eta = 2$ in order to obtain the denoising score matching loss equivalence (\ie, $L, L'$ are equivalent in \eqref{app:eq:lim:denoising_loss}).
One must make the assumption that the losses $L, L'$ are not infinite for some $\theta$, which is not necessarily realistic because $S_t(X_t), S_t(X_t|X_0)$ are heavy-tailed random variables involving $\alpha$-stable noise, and as such admit no variance.  
    \item In the case of LIM, in the experiments the parameters $p=1$ and $\eta=1$ are chosen, instead of the previous squared loss, in order to obtain more stable training, potentially indicating that indeed $L, L'$ \eqref{app:eq:lim:denoising_loss} might be infinite.
\end{itemize}


\section{Technical Results}\label{app:proofs}
In this section, we give the proofs relative to our technique for faster training, as introduced in \Cref{app:less_rv}.

\begin{lemma}
\label{app:lemma:reduction}

Let $\barA[0], \barA[1]$ bet two independent random variables distributed as $\stableA$. Define $Z_0 = Y_0$, and
\begin{equation}
    Z_t = \gall Z_0 + \sall \left(\barA[1]\right)^{1/2} G_t\eqsp.
\end{equation}
Moreover, let $Z_{t-1}$ be equal to 
 \begin{equation}
        Z_{t-1} = \frac{1}{\gstep}\left( Z_{t} - \GammaP \sall \epsilon_{t}(Z_t, Z_0) \right) + \SallPt G_{t-1}\eqsp,
    \end{equation}
where
\begin{align}
\SallPt &= \GammaP \sall[t-1] \left(\barA[0]\right)^{1/2}\\ 
    \GammaP &= \dfrac{\barA[1] \sstep^2} {\barA[1] \sstep^2 + \gstep^2 \sall[t-1]^2 \barA[0]} \\  \epsilon_{t}(Z_t, Z_0) &= \frac{Z_{t} - \gall Z_0}{\sall}\eqsp.
\end{align}
Then the joint distribution of $(Z_0, Z_{t-1}, Z_t, \barA[1], \barA[0])$ matches the joint distribution of 
\begin{equation}
    \left(Y_0, Y_{t-1}, Y_t, A_t, \frac{\SallPt}{ \sall^2}\right)\eqsp.
\end{equation}


\end{lemma}

\begin{proof}
    Consider the setting of \Cref{app:prop_backward}. The distribution of $Y_{t-1}$ given $Y_t, Y_0, A_{1:T}$ is characterized by the values of $\Sall, \Gamma_t$:
    \begin{equation}
        \meank_{t-1} = \frac{1}{\gstep}\left( Y_{t} - \Gamma_t \sall \epsilon_{t}(Y_t, Y_0) \right), \quad 
        \Sigmak_{t-1}(A_{1:t}) = \GammaA \SallA[t-1]\eqsp,
    \end{equation}
    where 
    \begin{align}
        \SallA[t] &= \sstep^2 A_t + \gstep^2 \SallA[t-1] \\ 
        \epsilon_{t}(Y_t, Y_0) &= \frac{Y_{t} - \gall Y_0}{\sall} \\ 
        \GammaA &= 1 - \frac{\gstep^2 \SallA[t-1]}{\SallA[t]}\eqsp.
    \end{align}
    Directly applying the result of \Cref{app:distrib_Sall}, we can affirm that
    \begin{equation}
        \SallA[t-1]\eqd \sall[t-1]^2 \barA[0]\eqsp,
    \end{equation}
    where $\barA[0] \sim \stableA$. In this conditions, the distribution of $\GammaA$ is equal to that of $\GammaP$, where 
    \begin{equation}
        \GammaP = 1 - \frac{\gstep^2 \sall[t-1]^2 \barA[0]}{\sstep^2 A_t + \gstep^2 \sall[t-1]^2 \barA[0]} 
    \end{equation}
    Since the distribution of $Z_{t}$ does not change if we draw another independent $\barA[1]$ instead of $A_t$, this ends the proof.
\end{proof}

\begin{lemma}[Sampling $\Sall$ with a single heavy-tailed r.v] \label{app:distrib_Sall}
Consider the setting of the data augmentation approach in \Cref{data_augmentation_approach}, where in particular $\{A_{t}\}_{t=1}^T$ are independent random variables distributed according to $A_t \sim \stable^1_{\alpha/2,1}(0, c_A )$, with $c_A =  \cos^{2/\alpha}(\uppi\alpha/4)$. Consider the random variable $\SallA[t]$, as defined in \eqref{app:eq:gall_Sall}: 
\begin{equation} 
\SallA[t] = \sum_{k = 1}^t \left(\frac{\gall[t]}{\gall[k]}\sqrt{A_k}\sstep[k]\right)^{2}\eqsp.
\end{equation}
Then
\begin{equation}
    \Sall[t](A_{1:t}) \eqd \ \sall[t]^2A,
\end{equation}
where $A\sim \stable^1_{\alpha/2,1}(0, c_A )$.
\end{lemma}

\begin{proof} 
By \Cref{app:gall_Sall}, $Y_t$ given $Y_0, A_{1:t}$ is a random variable distributed as a Gaussian of variance $\SallA[t]$: 
\begin{equation}
    Y_t \eqd \gall[t]Y_0 + \sqrt{\SallA[t]}\bar G_t\eqsp,
\end{equation}
where $\bar G_t$ is distributed as a standard Gaussian. Remember that $Y_t$ and $X_t$ admit the same distribution, with $X_t =\eqd \gall[t]X_0 + \bar \epsilon_t$ where $\bar \epsilon_t$ is distributed as a $\stablesym$.

In the same spirit we can define a third sequence of random variables $\{Z_t\}_{t=0}^T$ with $Z_0 = X_0$, and 
\begin{equation}
Z_t = \gall[t] Z_0 + \sall[t] \sqrt{A_t'}\bar G_t\eqsp,    
\end{equation} 
where $\{A_t'\}_{t=0}^T$ are independent random variables distributed according to $A_t' \sim \stable^1_{\alpha/2,1}(0, c_A )$. It is then quite clear from \Cref{subsec:heavy-tailed} that $Z_t$ and $Y_t$ admit the same distribution; in particular, 
\begin{equation}
    \sqrt{\Sall[t]}G_t \eqd \sall[t] \sqrt{A_t'}G_t'\eqsp.
\end{equation}
From there, we use \Cref{app:equating_variance} to conclude that $\sqrt{\Sall[t]} \eqd \sall[t] \sqrt{A_t'}$, which ends the proof.
\end{proof}

\begin{lemma}{} \label{app:equating_variance}
Let $A, A'$ be positive real random variables, let $Z$ be a real continuous random variable with density $p_Z$. Suppose that $AZ$ and $A'Z$ admit the same distribution. Then $A$, $A'$ admit the same distribution too.
\end{lemma}

\begin{proof}
    Let $h$ be a measurable function. Then $\mE(h(A)) = \mE(h(AZ/Z)) = \mE(h(A'Z/Z)) = \mE(h(A'))$. This shows that $A, A'$ have the same distribution.
\end{proof}

\section{Additional Experimental Details}
\label{app:experiment}
All experiments are conducted using PyTorch. 
We use linear timesteps during training and sampling, and the scale-preserving process\footnote{we mention again that it is traditionally called the variance preserving process}, being the only forward process readily provided by LIM. This entails choosing a sequence $\{\beta_t\}_{t=1}^T$ such that 
\begin{equation}
\gstep = (1 - \beta_t)^{1/\alpha}, \quad \sstep = (1 - \gstep^{\alpha})^{1/\alpha}\eqsp,
\end{equation}
resulting in $\sall = (1 - \gall^{\alpha})^{1/\alpha}$ and $\gall = \prod_{i = 1}^t \gstep[i]$. 
With this choice, we obtain approximately $X_T \sim \stable_{\alpha}(0, \Idd)$. We choose $\{\beta_i\}_{i=1}^T$ as the cosine schedule, as introduced by \cite{nichol2021improved}. 

We do not give any of the heavy-tailed random variables $\listTA$ as input to the neural network architecture, as we have witnessed worse performance in every scenarios we tried: learned embedding added to each model layer, concatenation to model input, concatenation at each layer, or feeding $\log(A_{1:T})$ instead of $A_{1:T}$ to better manage large jumps. This corresponds to the design choice \hyperlink{D2}{D2} in \Cref{app:dlpm:loss:intuitive}.

For image data generation with LIM, we use the same clipping hyper-parameters specified in \cite{NEURIPS2023_score_based_levy_lim}.

All the training and experiments are conducted on four NVIDIA RTX8000 GPU and four NVIDIA V100 GPU, where a single training run on MNIST or CIFAR10\_LT takes approximately 1 day per GPU, and requires about 4-12GB of VRAM for the batch sizes we use. Generating 5000 images with 1000 backward steps takes approximately 3-4 hours on one RTX8000 GPU.



\subsection{2D Data}\label{app:exp:2d}

We give more details about the mixture of Gaussian we consider in our experiment. It is designed in a grid-like pattern in $[-1, 1]^2$, as follows:
\begin{equation}
    \sum_{i=1}^9 w_i \cdot \normal(\mu_i, \sigma^2 \mathrm{I}_2)\eqsp,
\end{equation}
where $(w_i)_{i=1}^9 = (0.01, \ 0.1, \ 0.3, \ 0.2, \ 0.02, \ 0.15, \ 0.02, \ 0.15, \ 0.05)$, $\mu_i = (\mu_1, \mu_2)$ with $\mu_1 = (i \ \text{mod} \ 3) - 1, \ \mu_2 = \lfloor i / 3 \rfloor -1$, and $\sigma = 0.05$.

For our 2D datasets, we use 32000 datapoints for training, a batch size of 1024, and 25000 points for evaluation. We train each model for 10000 steps. Since we do not focus on the effect of diffusion steps, we set it to 100, where all methods have been observed to perform optimally.

The optimizer is Adam (\cite{kingma2017adam}) with learning rate 5e-3.
We use a neural network consisting of four time-conditioned MLP blocks with skip connections, each of which consisting of two fully connected layers of width 64. The time $t$ passes through two fully connected layers of size 32x32, and is fed to each time conditioned block, where it passes through an additional 32x64 fully connected layer before being component-wise added to the middle layer.

We compute a mean squared logarithmic error (MSLE) loss, designed to assess the fit to tails of distributions. Since it depends on the one-dimensional cumulative distribution function, we calculate it after projecting the data onto each dimension. In this simple setting, we keep the score computed on the first dimension.

We also compute the precision/recall metrics, as presented in \Cref{app:metrics}.

\subsection{Image data} \label{app:exp:image}
We work on the MNIST and the CIFAR10\_LT dataset. CIFAR10\_LT consists of the CIFAR10 images were artificial class unbalance has been introduced. The specific class counts we use are $[5000,2997,1796,1077,645,387,232,139,83,50]$.

The optimizer is Adam (\cite{kingma2017adam}) with learning rate 1e-3 for MNIST and 2e-4 for CIFAR10\_LT. We use the StepLR scheduler which scales the learning rate by $\gamma= .99$ every $N = 1000$ steps for CIFAR10\_LT and $N=400$ for MNIST.

To establish a fair comparison, LIM and DLPM use the same network model. We use a U-Net following the implementation of \cite{nichol2021improved} available in \url{https://github.com/openai/improved-diffusion}. We dimension the network as follows:
we set the hidden layers to $[128, 256, 256, 256]$, fix the number of residual blocks to 2 at each level, and add self-attention block at resolution 16x16, using 4 heads. We use an exponential moving average with a rate of 0.99 for MNIST and 0.9999 for CIFAR10\_LT. We use the silu activation function at every layer. Diffusion time $t$ is rescaled to $(0, 1)$ and fed to the model through the Transformer sinusoidal position embedding (\cite{vaswani2023attention}). We train MNIST for 120000 steps with batch size 256 with a time horizon $T=1000$, and CIFAR\_LT for 400000 steps with batch size 100 with a time horizon $T=4000$.

We use the FID metric for assessing the quality of our generative models, computing this metric between 5000 using images and 5000 generated images.

\subsection{Metrics for generative models} \label{app:metrics}

\paragraph{MSLE} we use a mean squared logarithmic error (MSLE) metric tailored to measure the fit on the tails of the distribution at hand.
 Drawing inspiration from \cite{allouche:ev_gan_heavy_tail}, we define the MSLE as the squared distance between the logarithm of the inverse cumulative distributions of the original and generated data. If $F, \hat F$ denote respectively the cumulative distribution function of the original data and the empirical cumulative distribution function of the generated data, then 
\begin{equation}
    \text{MSLE}(\xi) = \int_{\xi}^1 \left(\log F^{-1}(p) - \log \hat F^{-1}(p) \right)^2 dp\eqsp,
\end{equation}
where $\xi$ is chosen the be $0.95$.

\paragraph{Precision/recall} These metrics are introduced in the setting of generative models by \cite{sajjadi2018assessing}, and assess the overlap of sample distributions using local geometric structures. Precision measures how much the generated distribution is contained in the original data distribution (measuring quality), and recall measured how much of the original data distribution is covered by the generated distribution (diversity). We also consider the $\fpr$ score which we define as the harmonic mean of these two values:
\begin{equation}
    \fpr = 2\cdot \dfrac{\text{precision}\cdot \text{recall}}{\text{precision} + \text{recall}}\eqsp.
\end{equation}


%

\subsection{Additional results}
\label{app:exp:additional}
We provide some more results on MNIST and CIFAR10\_LT, with FID for non-isotropic noise, and with the $\fpr$ metric for other methods (with clipping enabled in LIM and LIM-ODE). We also provide grid images in order to visually check the performance of DLPM.

\begin{minipage}[t]{\textwidth}
    \centering
    \begin{tabular}{lccccc}
        \toprule
        MNIST & $\alpha=1.5$ & $\alpha=1.7$ & $\alpha=1.8$ & $\alpha=1.9$ & $\alpha=2.0$ \\
        \midrule  
        DLPM$^{\text{ni}}$ & 44.17   & 14.06   & 5.74    & 3.62    & - \\
        DLIM$^{\text{ni}}$ & 14.96   & 51.58   & 59.84   & 76.03   & - \\
        \bottomrule
    \end{tabular}
    \captionof{table}{FID$\downarrow$, 1000 sampling steps for DLPM$^{\text{ni}}$, 25 sampling steps for DLIM$^{\text{ni}}$.}
    \label{tab:mnist_ni}
\end{minipage}

\begin{table}[ht]
    \centering
    \scalebox{1.0}{
    \begin{tabular}{lccccc}
    \toprule
     &   DLIM &   DLPM &    LIM & LIM-ODE & DDPM \\
    $\alpha$ &        &        &        &  &       \\
    \midrule
    1.7   &  0.884 &  \textbf{0.887} &  0.857 &   0.869 & - \\
    1.8   &  0.874 &  \textbf{0.881} &  0.821 &   0.875 & - \\
    1.9   &  0.877 &  \textbf{0.878} &  0.700 &   0.808 & - \\
    2.0   &  0.820 &  0.871 &  0.694 &   0.772 & \textbf{0.881} \\
    \bottomrule
    \end{tabular}
    }
\caption{MNIST, $\fpr \uparrow$}
\end{table}
\begin{table}[ht]
    \centering
    \scalebox{1.0}{
    \begin{tabular}{lccccc}
    \toprule
     &   DLIM &   DLPM &    LIM & LIM-ODE & DDPM \\
    $\alpha$ &        &        &        &      &   \\
    \midrule
    1.7   &  0.676 &  0.675 &  \textbf{0.679} &   0.677 & - \\
    1.8   &  0.669 &  \textbf{0.680} &  0.677 &   0.673 & - \\
    1.9   &  0.667 &  \textbf{0.669} &  0.661 &   0.669 & - \\
    2.0   &  0.664 &  \textbf{0.667} &  0.660 &   0.665 & 0.666 \\
    \bottomrule
    \end{tabular}
    }
    \caption{CIFAR10\_LT, $\fpr \uparrow$}
\end{table}
\begin{figure}[h!]
  \centering
  \begin{minipage}{0.6\textwidth}
    \centering
    \includegraphics[width=\linewidth]{img_rebuttal/img_grid_sto/mnist/mnist_ddpm_1.7.png}
    \caption{MNIST, DLPM ($\alpha = 1.7$)}
  \end{minipage}
\end{figure}
\begin{figure}[h!]
  \centering
  \begin{minipage}{0.6\textwidth}
    \centering
    \includegraphics[width=\linewidth]{img_rebuttal/img_grid_sto/cifar10_lt/cifar10_lt_ddpm_1.7.png}
    \caption{CIFAR10\_LT, DLPM ($\alpha = 1.7$)}
  \end{minipage}
\end{figure}

\end{document}